%% file: main.tex
\documentclass[sigconf,balance=false]{acmart}
\usepackage{popets}

\setcopyright{popets}
\copyrightyear{YYYY}

\acmYear{YYYY}
\acmVolume{YYYY}
\acmNumber{X}
\acmDOI{XXXXXXX.XXXXXXX}
\acmISBN{}
\acmConference{Proceedings on Privacy Enhancing Technologies}
\usepackage{algorithmic}
\usepackage{graphicx}
\usepackage{todonotes}
\usepackage{xspace}
\usepackage{amsthm}

\newcommand{\orig}{\textsc{Original}\xspace}
\newcommand{\alledges}{\textsc{FullGraph}\xspace}
\newcommand{\nograph}{\textsc{NoEdges}\xspace}
\newcommand{\gcn}{\textsc{Gcn}\xspace}
\newcommand{\gat}{\textsc{Gat}\xspace}

\newcommand{\sage}{\textsc{GraphSage}\xspace}
\newcommand{\cora}{\textsc{Cora}\xspace}

\newcommand{\pubmed}{\textsc{PubMed}\xspace}
\newcommand{\chameleon}{\textsc{Chameleon}\xspace}

\newtheorem{experiment}{Experiment}

\usepackage{subcaption}    
\usepackage{caption}
\usepackage{amsmath}

\usepackage{makecell}
\usepackage{multirow}
\usepackage{rotate}
\usepackage{xcolor}
\usepackage{enumitem}
\usepackage{booktabs}

\begin{document}

\title{Impact of Graph Structure on Membership-Inference Risk for Graph Neural Networks}


\author{Megha Khosla}
\orcid{0000-0002-0319-3181}
\affiliation{%
  \institution{Delft University of Technology}
  \city{Delft}
  \country{The Netherlands}}
\email{m.khosla@tudelft.nl}


\renewcommand{\shortauthors}{Khosla}

\begin{abstract}
  Graph neural networks (GNNs) are widely used for tasks such as node classification and link prediction, but their use in sensitive settings raises concerns about training-data leakage. Prior work on privacy leakage in GNNs largely borrows assumptions from non-graph domains, overlooking the role of graph structure. We argue for a graph-specific analysis of privacy risk and study how graph structure affects node-level membership inference. We  formalize membership inference (MI) over node–neighborhood tuples and  investigate two important dimensions: (i) training-graph construction and (ii) inference-time edge access. 
 We compare \emph{snowball sampling}, a structure-aware procedure, with \emph{uniform random node sampling} for constructing training graphs.
  Our experiments show that snowball sampling often hurts generalization relative to random sampling due to its coverage bias. In contrast, allowing access to inter–train–test edges at inference improves test accuracy, reduces the train–test gap, while also having a strong and setting-dependent effect on membership advantage. These results show that graph structure directly shapes privacy risk. We further show that the generalization gap, measured as the performance difference between training and test nodes, is an \emph{incomplete} proxy for membership inference risk: membership advantage can rise or fall independently of changes in this gap, with inference-time edge access often playing a crucial role. Theoretically, we show that for node-level tasks, standard privacy-auditing results based on membership inference do not directly carry over to inductive graph settings, because training and test nodes are structurally dependent rather than interchangeable. We release the code and data at \url{https://github.com/PriXAI/GraphStructurePrivacyAnalysis-public}.

\end{abstract}


\maketitle

\input{introduction}
\input{background}
\input{threatmodel}

\input{experiments}

\input{relatedwork}

\input{discussion}

\begin{acks}
This publication is part of the project PriXAI with file number VI.Vidi.243.224 of the research programme Vidi ENW which is financed by the Dutch Research Council (NWO) under the grant https://doi.org/10.61686/FATII77836.
The author used generative AI-based tools to revise parts of the text, improve flow and correct any typos, grammatical errors, and awkward phrasing.
\end{acks}

\bibliographystyle{ACM-Reference-Format}
\bibliography{references}

\appendix

\input{appendix}

\end{document}

%% file: introduction.tex
\section{Introduction}
\label{sec:intro}
Graph Neural Networks (GNNs) are widely employed for learning on relational data, delivering state-of-the-art results in chemistry, recommender systems, knowledge graphs, and more \cite{gaudelet2021utilizing,schulte2021integration,sanchez2020learning}. However on the downside, GNNs have been shown to leak private information about the very graphs they are trained on \cite{duddu2020quantifying,conti2022label,he2021stealing,jnaini2022powerful,wang2024subgraph,wu2021adapting,dai2022comprehensive,wang2021membership,olatunji2021membership}. Existing privacy risk studies for GNNs mainly adapt techniques from the i.i.d. scenarios \cite{shokri2017membership}, treating each example as independent and largely ignoring the graph structure that connects data points together.

We focus on membership inference (MI) risk: given access to a trained model, an adversary aims to determine whether a specific data instance was included in its training set. Studying membership inference is important not only because membership itself may constitute sensitive information, but also because it serves as a proxy for stronger privacy violations such as attribute inference or data reconstruction attacks. Intuitively, if an adversary cannot reliably determine whether a given instance participated in training, then reconstructing hidden properties or recovering parts of the training data becomes substantially harder.

In graph-based machine learning using graph neural networks, what is considered as a data instance depends on the task: node-level tasks view each node as an instance (e.g., predicting a protein’s function in a PPI network), edge-level tasks treat each edge as an instance (e.g., recommending a friendship), and graph-level tasks consider the entire graph as a single instance (e.g., classifying a molecule’s toxicity). Our study focuses on node-level tasks, where instances (being nodes) are inherently interdependent due to the existence of edges between them.  This inter-instance dependency is largely absent from prior privacy analyses and leads us to ask: 

\begin{quote}
\emph{How strongly does the amount of edge information available during training and at inference modulate node-level membership-inference risk?}
\end{quote}
\paragraph{Training Graph Construction.} To investigate this question, we start from the
inductive setting of \cite{olatunji2021membership}, where training and test node sets are disjoint
and, during training, neither test nodes nor any train–test edges are available. A common way to
instantiate such splits is to sample train/test nodes uniformly at random and use the induced
subgraphs. While this preserves  independence among sampled nodes, it is often
unrealistic. For instance, random node sampling can produce isolated nodes when none of their neighbors are included, which undermines the rationale for graph-based learning.
\paragraph{Structure Based Sampling and Coverage Bias} A more plausible  sampling strategy is \emph{snowball sampling}, in which, starting from random seed nodes, we iteratively add a fixed number of randomly chosen neighbors until the desired training size is reached. Variants of snowball sampling are quite popular in real-world scenarios such as for building social-network graphs \cite{mislove2007measurement} and for contact tracing in healthcare \cite{kennedy2021snowball}. Specifically following a snowball procedure help us in generating a well-connected training graph. Nevertheless, the generated sample is coverage-biased in the sense that the expansions tend to remain
within a few regions, oversampling high-degree hubs while leaving peripheral or distinct communities
under-represented. The resulting training distribution is therefore non-i.i.d.\ with respect to the
full graph, motivating our first axis of inquiry:
\emph{How does biased coverage affect generalization to held-out nodes and, in turn, membership-inference risk?}
\subsubsection*{Edge Access at Inference Time} A second graph-specific axis is \emph{edge access at inference}. In our node--neighborhood formulation, each query supplies a subgraph, and a GNN’s prediction for a given node depends on that subgraph at query time. Even with a fixed train--test node split, adding or withholding edges can substantially alter the outputs of a frozen model. Unlike i.i.d.\ domains, an adversary may legitimately provide relational context when probing membership. We motivate this axis further through a practical example. Consider a disease-spread network constructed via contact tracing, where investigators start from an initial set of patients and expand the trace by following contacts. In such settings, even if an adversary has access to the broader population cohort, they may not know the realized tracing pattern, either partially or completely.
We therefore ask: 
\emph{How does the adversary’s edge knowledge modulate MI success?}

In prior literature MI success is often linked to overfitting and generalization error \cite{yeom2018privacy}.
Because the true data-generating distribution is unknown, practitioners use the empirical
train–test \emph{gap} as a surrogate. Yet prior work suggests that attacker performance does not
always increase with this empirical gap in graph settings \cite{olatunji2021membership}.
Here, the gap itself depends not only on which nodes are selected for train/test but also on which
neighborhoods are presented at inference. This leads to our third guiding question: \emph{What is the relationship between the generalization gap and membership-inference risk across training-graph constructions (random vs.\ snowball) and inference-time edge regimes?}
\subsubsection*{Exchangeability and Privacy Auditing.} Finally, while we empirically studied membership advantage under different settings, we do not attempt to translate these results into a formal lower bound on the differential privacy budget of GNNs or to audit differential private GNNs, as is common in literature on privacy auditing \cite{kazmi2024panoramia}. The simple reason for that is such a formal connection cannot be always guaranteed under the differential privacy notion and when considering specific training graph construction schemes for inductive learning settings. To support our arguments, we adapt the definition of \emph{statistical exchangeability} from \cite{humphries2023investigating} to GNNs by treating each example as a node–neighborhood tuple, and show that inductive train/test constructions, whether via random node sampling or snowball sampling, break exchangeability because neighborhoods are restricted to, and dependent on, the sampled training subgraph. This breaking of statistical exchangeability (as shown in Section \ref{sec:theory}) implies that the data-generating distribution may already leak information about node membership.
In such cases, differential privacy, being oblivious to the data generating distribution, may not be auditable under practical node-level membership inference attacks.



%% file: background.tex
\section{Background}
We provide background on the graph sampling strategies, Graph Neural Networks (GNNs), and membership advantage used throughout this work.
\subsection{Graph Sampling Strategies} We employ two graph sampling strategies: \textit{random sampling} and \textit{snowball sampling}. These strategies determine not only which nodes are included in the training set but also the resulting graph structure, which in turn influences both model learning and the adversary’s ability to infer membership. We chose random sampling as it is the simplest and the most widely used strategy in academic literature to construct train graphs for training graph neural networks. In contrast, as discussed in Section~\ref{sec:intro}, structured sampling techniques such as snowball sampling are more commonly employed when constructing graphs from real-world data, especially in building social science surveys and in healthcare studies.

\textbf{Random Sampling.} Formally, given a graph $\mathcal{G} = (\mathcal{V}, \mathcal{E})$, where $\mathcal{V}$ is the set of $n$ nodes and $\mathcal{E}$ is the set of $m$ edges, we randomly sample $N$ nodes independently and uniformly to form the sampled node set $\mathcal{V}_S$. The training graph is then constructed by including an edge $(u,v)$ if and only if both $u, v \in \mathcal{V}_S$ and $(u,v) \in \mathcal{E}$.

\textbf{Snowball Sampling.} Here we start from an initial set of nodes $\mathcal{V}^{(0)}$ (in our experiments we randomly chose 10 nodes from each class to form $\mathcal{V}^{(0)}$). At each stage $i$ we choose $k$ neighbors of each of the nodes in $\mathcal{V}^{(i-1)}$. This is equivalent to obtaining a sample of incident edges of $\mathcal{V}^{(i-1)}$ which we denote by $\mathcal{E}^{(i)}$. We construct $\mathcal{V}^{i}$ by adding only the new nodes discovered in the last stage. The process continues until the maximum number of nodes are sampled. The sampled train graph then is $\mathcal{G_S}=(\mathcal{V}_S, \mathcal{E}_S)$ where $\mathcal{V}_S=\cup_{i}\mathcal{V}^{(i)}$ and $\mathcal{E}_S=\cup_{i}\mathcal{E}^{(i)}$

\subsection{Graph Neural Networks}
\label{sec:GNNs}
Graph Neural Networks (GNNs) are a class of machine learning models designed specifically to learn on graph-structured data. The typical input to a GNN consists of the graph $\mathcal{G}=(\mathcal{V},\mathcal{E})$ where $\mathcal{V}$ and $\mathcal{E}$ denotes the set of nodes and edges (connecting pairs of nodes) respectively. In addition the node features matrix $\mathbf{X} = (\mathbf{x}_{1}^\top, \mathbf{x}_{2}^\top, \ldots, \mathbf{x}^\top_{n})$ is provided where $\mathbf{x}_{i}$ is the input feature vector for node $i$. In the supervised setting, each node in the training set is associated with a one-hot encoded label vector $\mathbf{y}_i$, and the model is trained to predict these labels for unseen nodes.

At each GNN layer, a node's representation is updated by aggregating its own features along with the features of its neighbors. After this aggregation step, the node's representation is transformed using a non-linear function. More formally, let $\mathbf{x}_{i}^{(\ell)}$ denote the feature representation of node $i$ at layer $\ell$, and let $\mathcal{N}_{i}$ be the set of nodes in the 1-hop neighborhood of node $i$. The graph convolution operation at layer $\ell$ is defined as follows:

\begin{align}
\label{eq:aggre}
\mathbf{z}_{i}^{(\ell)}=&\operatorname{AGGREGATION}^{(\ell)}\left(\left\{\mathbf{x}_{i}^{(\ell-1)},\left\{\mathbf{x}_{j}^{(\ell-1)} \mid j \in{\mathcal{N}}_{i}\right\}\right\}\right) \\
    \mathbf{x}_{i}^{(\ell)}= &\operatorname{TRANSFORMATION} ^{(\ell)}\left(\mathbf{z}_{i}^{(\ell)}\right)
\end{align}
At the final layer (denoted $L$), a softmax function is applied to the node representations to produce the predicted class probabilities for each node. This can be expressed as:
\begin{align}
\label{eq:pred}
     \hat{\mathbf{y}}_i\leftarrow \operatorname{softmax}(\mathbf{z}_{i}^{(L)}\mathbf{\Theta}),
\end{align}
Here, $\hat{\mathbf{y}}_i \in \mathbb{R}^{c}$, where $c$ is the number of classes, and $\mathbf{\Theta}$ is a learnable weight matrix. The $jth$ element $\hat{\mathbf{y}}_i$ represents the predicted probability that node $i$ belongs to class $j$. We focus on three representative aggregation operations in GNNs, which serve as the basis for more complex aggregation methods. The concrete realizations of the aggregation and transformation operations for the three GNN architectures employed in this work are provided in Appendix~\ref{sec:aggre}.

\subsection*{Train--Test Paradigms}
Graph-based learning typically distinguishes between two train--test paradigms depending on whether test nodes are visible during training.
In the \textbf{transductive} setting, the full graph $\mathcal{G}$ is available during training, although labels are only known for the train nodes. Consequently, test nodes can still influence learned representations through message passing and neighborhood aggregation. By contrast, in the \textbf{inductive} setting, train and test nodes are disjoint, and neither test nodes nor train--test edges are exposed during training.
Following \cite{olatunji2021membership}, we focus on the inductive setting, where the distinction between member and non-member nodes remains unambiguous.
\subsection{Why Graph Structure Matters for GNN Learning and Privacy}

We note that the  aggregation mechanism of the GNNs (as depicted in \eqref{eq:aggre}) makes the learned representation of a node directly dependent on the graph structure. Specifically, the representation of a node is influenced not only by its own features, but also by the features and connectivity of nodes in its neighborhood. Consequently, changes in the neighborhood structure, such as adding or removing edges and neighboring nodes, modify the information aggregated by the model and can therefore alter the resulting node representations and predictions.

\subsubsection{Availability and Composition of Node Neighborhoods}

The above discussed dependence on structure is relevant during both training and inference. During training, graph sampling strategies determine which neighborhoods are visible to the model, which nodes repeatedly participate in aggregation, and what structural patterns the model learns to rely on. As shown in Section~\ref{sec:sampled_nodes}, different sampling strategies can substantially alter the structural properties of the resulting training graph. For example, snowball sampling tends to bias the sample toward high-degree nodes because such nodes are more likely to be reached during recursive neighborhood expansion. In contrast, under uniform random node sampling, all nodes are sampled with equal probability, but the observed degree distribution in the induced subgraph shifts toward lower values because only edges between sampled nodes are retained. 

\subsubsection{Structural Properties and Representation Learning}

The differences in availability and composition of node neighborhoods induced by different sampling strategies as discussed above directly affect GNN behavior. For instance, high-degree nodes typically receive richer neighborhood information during aggregation, whereas isolated or sparsely connected nodes provide limited context. Similarly, nodes connected to neighbors with similar labels or features are generally easier for GNNs to learn from. Consequently, the structural properties of the sampled training graph directly determine the node--neighborhood patterns observed during training and therefore influence model generalization. When the sampled graph does not adequately represent the structural diversity of the original graph, the learned representations may generalize unevenly across different regions of the network. In particular, structure-based sampling strategies such as snowball sampling can over-represent specific graph regions and connectivity patterns, leading to uneven graph coverage and increasing the behavioral gap between member and non-member nodes, which can potentially make membership inference easier.

\subsubsection{Availability of Graph Structure at Inference Time}
Since neighborhood aggregation is performed also during inference, the edges and neighborhoods accessible at query time directly influence the predictions of a frozen GNN. Consequently, even for a fixed trained model, predictions for the same node can vary depending on the graph structure provided during inference. In Section~\ref{sec:node_neighborhoods}, we empirically analyze this effect by comparing the output distributions of the same trained model under different inference-time graph access regimes. We observe that changing only the neighborhood structure can significantly alter the output distributions produced by the model, which in turn can affect membership inference attack success.

\subsection{Membership Advantage}

To quantify privacy leakage under membership inference attacks, we follow \cite{yeom2018privacy} and use the notion of \emph{membership advantage}. Intuitively, membership advantage measures how well an adversary can distinguish members from non-members beyond random guessing. An attacker with no useful information should behave similarly on both groups, resulting in an advantage close to zero, whereas a larger advantage indicates stronger membership leakage.

Formally, let \(b \in \{0,1\}\) denote the membership status of a queried instance, where \(b=0\) indicates that the instance belongs to the training set and \(b=1\) indicates a non-member instance. Let \(\hat{b} \leftarrow \mathcal{A}(\cdot)\) denote the membership prediction of the adversary. The membership advantage of \(\mathcal{A}\) is defined as
\[
\operatorname{Adv}(\mathcal{A})
=
\Pr(\hat{b}=0 \mid b=0)
-
\Pr(\hat{b}=0 \mid b=1).
\]
that is, the difference between the true positive rate and false positive rate of the attacker.

In practice, the attacker outputs a membership score \(s(x)\in [0,1]\), representing the predicted likelihood that an input belongs to the training set. Given a threshold \(\tau\), the attacker predicts membership whenever \(s(x)\geq \tau\). The empirical membership advantage at threshold \(\tau\) is therefore defined as:
\[
\widehat{\operatorname{Adv}}_{\tau}(\mathcal{A})
=
\widehat{\operatorname{TPR}}(\tau)
-
\widehat{\operatorname{FPR}}(\tau),
\]
where \(\widehat{\operatorname{TPR}}(\tau)\) and \(\widehat{\operatorname{FPR}}(\tau)\) denote the empirical true positive rate and false positive rate of the attacker at threshold \(\tau\), respectively.

We then compute the empirical membership advantage as the maximum difference between the true positive rate and false positive rate across all thresholds:
\[
\widehat{\operatorname{Adv}}(\mathcal{A})
=
\max_{\tau}
\left(
\widehat{\operatorname{TPR}}(\tau)
-
\widehat{\operatorname{FPR}}(\tau)
\right).
\]
Equivalently, this corresponds to the maximum vertical separation between the ROC curve of the attacker and the diagonal random-guessing baseline.

%% file: threatmodel.tex
\section{Node-level Membership Inference}
\label{sec:nodeMI}
We study the problem of node-level membership inference in Graph Neural Networks (GNNs), where an adversary aims to determine whether a queried node was part of the training graph used to train a target model.
To formalize node-level membership inference in GNNs, we first explicitly define the input instance corresponding to a node \(v\). Let \( \mathbf{x}_v \in \mathbb{R}^d \) denote the feature vector of node \(v\), \( y_v \in \mathcal{Y} \) its label, and \( \mathcal{N}^L_{\text{strategy}}(v) \) its \(L\)-hop neighborhood computed according to a neighborhood selection strategy (e.g., induced neighborhood, sampled neighbors, or full graph). Given a training graph \( \mathcal{G}_S = (\mathcal{V}_S, \mathcal{E}_S) \sim \mathcal{S}(\mathcal{G}, n) \), constructed by sampling \(n\) nodes from the original graph \( \mathcal{G} \) according to a sampling strategy \( \mathcal{S} \), the node-level input representation presented to the model is defined as
\begin{equation}
\label{eq:input}z_v^{\mathcal{G}_S}
:=
\left(
\mathbf{x}_v,
y_v,
\mathcal{N}^L_{\text{strategy}}(v),
\left\{
\mathbf{x}_u
:
u \in \mathcal{N}^L_{\text{strategy}}(v)
\right\}
\right). 
\end{equation}


\subsection{Node-Level Membership Inference Experiment}

We now define the node-level membership inference (MI) experiment while making explicit the role of graph sampling and neighborhood construction. For notational convenience, we restrict ourselves to multi-class settings where each node $v$ is associated with a single ground truth label $y_v$.

\begin{experiment}
(Node-Level Membership Inference~$\operatorname{Exp}(\mathcal{A}, \Phi, n, \mathcal{S})$)
\label{exp:nodelevel}
Let \( \mathcal{A} \) be an adversary, \( \Phi \) a learning algorithm, \( n \in \mathbb{N} \) a sample size, and \( \mathcal{S} \) a graph sampling strategy.

\begin{enumerate}

    \item Sample a subgraph \( \mathcal{G}_S = (\mathcal{V}_S, \mathcal{E}_S) \sim \mathcal{S}(\mathcal{G}, n) \).

    \item Train a GNN model \( \Phi_S = \Phi(\mathbf{X}_S, \mathcal{G}_S) \).

    \item Choose a bit \( b \leftarrow \{0,1\} \) uniformly at random.

    \item If \( b=0 \), sample a member node \( v \sim \mathcal{V}_S \); otherwise sample a non-member node \( v \sim \mathcal{V}\setminus\mathcal{V}_S \).

    \item Construct the node tuple
    \[
    z_v :=
    \left(
    \mathbf{x}_v,
    y_v,
    \mathcal{N}^{L}(v),
    \left\{
    \mathbf{x}_u : u \in \mathcal{N}^{L}(v)
    \right\}
    \right),
    \]
where \(\mathcal{N}^{L}(v)\) denotes the \(L\)-hop neighborhood of node \(v\) computed based on the adversary's graph knowledge \(\mathcal{G}_{Adv}\).

    \item The adversary outputs a membership prediction based on
    \[
    \mathcal{A}(\mathbf{x}_v, z_v, \Phi_S, \mathcal{G}_{Adv}).
    \]

\end{enumerate}

\end{experiment}

Given the trained model $\Phi_S$, target node $v$, and its corresponding node tuple $z_v$ constructed using adversarial graph knowledge from some graph $\mathcal{G}_{Adv}$, the goal of the adversary is to determine whether $v$ was part of the sampled training graph $\mathcal{G}_S$.

\subsection{Differences from Classical MI Settings}

Compared to classical membership inference attacks \cite{yeom2018privacy}, the graph setting introduces two important differences. First, unlike classical membership inference settings where inputs are independent feature vectors, the learning algorithm here operates on the node–neighborhood representations defined in \eqref{eq:input}.  Consequently, the sampling process used to construct the training graph itself becomes part of the MI experiment, as reflected in Step~1 and 2 of Experiment~\ref{exp:nodelevel}.

Second, the adversary’s capabilities depend on access to graph structure. Since GNN predictions depend on neighborhood aggregation, the queried input tuple constructed in Step~5 of Experiment~\ref{exp:nodelevel} explicitly includes the neighborhood $\mathcal{N}^L(v)$, constructed using adversary's knowledge of graph $\mathcal{G}_{Adv}$. Consequently, the validity and strength of the threat model directly depend on assumptions about the adversary’s knowledge of the underlying graph structure as also illustrated in the following example.

\paragraph{\textbf{Example}} A social scientist or healthcare provider may sample part of a population network to study voting behavior or disease spread and subsequently train a GNN on the collected subgraph. In such settings, the broader population network may itself be publicly or partially observable, while the specific subset of individuals selected into the study remains private. An adversary may therefore attempt to infer which individuals were included in the training graph. At one extreme, the adversary may have access to the full underlying graph structure and all node relationships. At the other extreme, the adversary may only observe node-level features without any knowledge of graph connectivity or neighborhood structure. In practice, however, realistic adversarial knowledge often lies somewhere in between these two extremes, where the adversary possesses only partial or noisy information about the neighborhood connectivity.

\input{DP}

%% file: DP.tex
\subsection{Statistical Implications of Neighborhood-Based Input }
\label{sec:theory}
An important distinction in graph settings is that, due to the explicit use of neighborhood information in GNNs (as reflected in Equation~\eqref{eq:input}), train and test instances are not always statistically exchangeable. Existing results that relate empirically measured membership advantage to lower bounds on a model’s differential privacy budget \cite{yeom2018privacy} rely on this exchangeability assumption. As discussed by \cite{humphries2023investigating}, such guarantees require train and test instances to be drawn from the same underlying distribution and remain interchangeable under permutation without altering the joint distribution. In graph domains, however, node representations depend on shared edges and overlapping neighborhoods, introducing structural dependencies between instances and thereby violating statistical exchangeability.

To support our claim we begin by reproducing and adapting the definition of statistical exchangeability from \cite{humphries2023investigating} to the setting of message passing graph neural networks (GNNs) as studied in this work.

\begin{definition}[Statistical Exchangeability in MI for GNNs]
\label{def:stats}
Let \( \mathcal{D} \) denote a joint distribution over \( n \) member samples \( \{z_1, \ldots, z_n\} \) and a non-member sample \( z_{n+1} \), where each sample \( z_v \) corresponds to a tuple:
\[
z_v := (\mathbf{x}_v, y_v, \mathcal{N}^{L}_{\mathcal{G'}}(v), \{\mathbf{x}_u : u \in \mathcal{N}^{L}_{\mathcal{G}'}(v)\}),
\]
with \( \mathbf{x}_v \) and $y_v$ being the feature vector of node $v$ and its label respectively, and \( \mathcal{N}^{L}_{\mathcal{G}'}(v) \) denoting the $L$-hop neighborhood of \( v \) in a graph view \( \mathcal{G}' \). The graph \( \mathcal{G}' \) may correspond to the full graph \( \mathcal{G} \), or to a subgraph induced by the sampled training node set with edges included according to a specific sampling scheme such as random or snowball sampling.

We say that \( \mathcal{D} \) is \emph{statistically exchangeable} if the joint distribution over all samples \( \{z_1, \ldots, z_{n+1}\} \sim \mathcal{D} \) is invariant under permutations of sample indices. That is, for any permutation \( \sigma: [n+1] \rightarrow [n+1] \), it holds that:
\begin{equation}
    \Pr(z_1, \ldots, z_{n+1}) = \Pr(z_{\sigma(1)}, \ldots, z_{\sigma(n+1)}).
\end{equation}
\end{definition}
In the following, we show that statistical exchangeability cannot, in general, be guaranteed in graph learning settings with GNNs. In particular, this assumption is violated in inductive settings under both random and snowball sampling.
\begin{theorem}
\label{thm:non_exchangeability}
Let \( \mathcal{D} \) be a joint distribution over \( n \) member samples \( \{z_1, \ldots, z_n\} \) and a non-member sample \( z_{n+1} \), where each sample \( z_v \) is defined as in Definition~\ref{def:stats}, and where neighborhoods \( \mathcal{N}^L_{\mathcal{G}'}(v) \) are computed over a graph view \( \mathcal{G}' \) constructed via a node sampling scheme. Then \( \mathcal{D} \) is not always guaranteed to be statistically exchangeable if \( \mathcal{G}' \) is a subgraph constructed over the a sampled set of training nodes \( S \subset V \), and \( \mathcal{N}^L_{\mathcal{G}'}(v) \) is restricted to lie within \( S \). 

\end{theorem}
\begin{proof}
We consider two inductive sampling schemes as studied in this work: random node sampling and snowball sampling.

\textbf{Case 1: Random Node Sampling.}
Let \( \mathcal{G}' \) denote the subgraph induced over the training node set \( S = \{x_1, \ldots, x_n\} \), and let the joint sample sequence be \( Z = (z_1, \ldots, z_n, z_{n+1}) \), where each \( z_i \) is defined as in Definition~\ref{def:stats}.

Now consider a permutation \( Z^\sigma = (z_{\sigma(1)}, \ldots, z_{\sigma(n+1)}) \) in which the non-member node \( x_{n+1} \) replaces some member node \( x_i \in S \). The resulting training set becomes:
\[
S' = \{x_1, \ldots, x_{i-1}, x_{n+1}, x_{i+1}, \ldots, x_n\},
\]
and we define the new induced subgraph as \( \mathcal{G}'' \). Since neighborhoods \( \mathcal{N}_\mathcal{G}^L(v) \) depend on the underlying graph, we have:
\[
\mathcal{N}_{\mathcal{G}'}^L(v) \ne \mathcal{N}_{\mathcal{G}''}^L(v) \quad \text{for some } v \in S \cup \{x_{n+1}\},
\]
unless both the exchanged nodes are isolated in the original graph $\mathcal{G}$.
Thus, the permuted tuple \( Z^\sigma \) becomes incompatible with the graph structure \( \mathcal{G}' \), and so:
\[
\Pr(Z^\sigma \mid \mathcal{G}') = 0, \quad \text{while} \quad \Pr(Z \mid \mathcal{G}') > 0.
\]
Since the graph view \( \mathcal{G}' \) is generated by the sampling process with non-zero probability, i.e., \( \Pr(\mathcal{G}') > 0 \), it follows that
the joint distribution is not invariant under permutation:
\[
\Pr(Z) \ne \Pr(Z^\sigma),
\]
violating statistical exchangeability.

\textbf{Case 2: Snowball Sampling.}
Let \( \mathcal{G}' \) be the subgraph generated using a snowball sampling process starting from a seed node set while sampling a fixed number of neighbors for each selected node in each iteration. In this setting, the set of sampled nodes and edges depends on the order in which nodes are traversed. 

Now consider a permutation \( Z^\sigma \) where the non-member node \( x_{n+1} \) replaces a member node \( x_i \). The resulting training set becomes:
\[
S' = \{x_1, \ldots, x_{i-1}, x_{n+1}, x_{i+1}, \ldots, x_n\}.
\]
Snowball sampling proceeds in layers by expanding neighbors of already sampled nodes. This means that whether a node is even included in the final subgraph depends on the specific sequence in which other nodes were included before it. Therefore, replacing \( x_i \) with \( x_{n+1} \) can not only alter neighborhoods but may also cause some nodes that were previously sampled to no longer appear in the new corresponding graph \( \mathcal{G}''\) , and vice versa.

As a result, the neighborhood \( \mathcal{N}_{\mathcal{G}''}^L(v) \) and even the support of \( Z^\sigma \) may differ fundamentally (in terms of node features) from that of \( Z \), and under a fixed \( \mathcal{G}' \), we have:
\[
\Pr(Z^\sigma \mid \mathcal{G}') = 0, \quad \text{while} \quad \Pr(Z \mid \mathcal{G}') > 0.
\]
Since the graph view \( \mathcal{G}' \) is generated by the sampling process with non-zero probability, i.e., \( \Pr(\mathcal{G}') > 0 \), it follows that
the joint distribution is not invariant under permutation:
\[
\Pr(Z) \ne \Pr(Z^\sigma),
\]
violating statistical exchangeability.
Thus, snowball sampling also violates statistical exchangeability due to its sequential and history-dependent nature.
\end{proof}

From the above discussion, we conclude that in inductive train-test splits, where the neighborhood structure of the training graph is determined by the set of sampled training nodes, statistical exchangeability between training and test samples cannot be always guaranteed. In contrast, in transductive settings, where the full graph $\mathcal{G}$ is accessible and fixed during training thus decoupling neighborhood structure from the sampled training nodes, statistical exchangeability holds when the training nodes are sampled randomly.

However, if snowball sampling is used to select training nodes, even when the full graph $\mathcal{G}$ is utilized during training, statistical exchangeability is still violated. This is because node inclusion under snowball sampling depends on the initial seed set and the traversal order, making the set of training nodes itself order-dependent. In other words, replacing even a single node in the sampled training set can make the resulting set structurally implausible under the sampling process.


%% file: experiments.tex
\section{Attack Model and Datasets }  
\label{sec:sampling}
We now describe the attack model and the datasets 
employed in our experimental evaluation.
\subsection{Attack Model}
\label{sec:attack}
Existing works \cite{duddu2020quantifying,olatunji2021membership,he2021node} typically employ a shadow model based \cite{shokri2017membership} strategy to build the attack model. In particular, a shadow model is trained to mimic the target model. The shadow model employs a shadow dataset usually assumed to be sampled from the same distribution as the training dataset of the target model. The attack model is designed as a binary classification model, which maps the output prediction probabilities of the shadow model on the shadow dataset $\mathcal{D}_{shadow}$ to the membership of the corresponding input member and non-member nodes. 

To isolate the impact of the shadow dataset quality and different shadow model training methods from the influence of graph structure on the performance of membership inference (MI), we entirely skip the step of building the shadow model. Instead, we directly train the attack model using the membership status of $r$ fraction of both the true member and non-member instances to the attackers, where $r\in \{0.1,0.2,0.5,0.8\}$, and evaluate membership prediction on the remaining nodes.

We employ a 2-layer multilayer perceptron with ReLU activation as our attack model. The input to the attack model is the output prediction vector of the target model corresponding to the query node concatenated with the cross entropy loss value computed over the query node. The attack model outputs the  probability that the queried node belongs to the training graph. Formally, let \( \mathbf{p}(v) \) represent the output prediction vector  for the query node \( v \), and let \( \mathcal{L}(v) \) denote the cross-entropy loss computed over \( v \). The input to the attack model corresponding to node $v$ is the concatenation of these two values, which can be written as:
\[
\mathbf{x}_{\text{attack}} = [\mathbf{p}(v) \, | \, \mathcal{L}(v)].
\]

\textbf{Remark 1.} Our goal is \emph{not} to propose a new attack but to quantify how \emph{edge structure}
(employed during training-graph construction and inference-time edge access) affects membership advantage.
The absolute success rate of any attack may change with attacker strength, model choices
or defences (e.g., DP, calibration). Our focus, however, is on the \emph{modulation} induced by edge
structure. For this purpose, our setup using true membership for a large amount of data suffices
to reveal the direction and relative magnitude of these edge-driven effects.

\textbf{Remark 2.} With slight abuse of terminology and for brevity we interchangeably refer to the prediction vector output by GNN for a node query as the \textbf{node's posterior}.

\subsection{Datasets}
We perform our investigations on three popular graph datasets, namely \cora \cite{sen2008collective}, \chameleon \cite{rozemberczki2019multiscale} and \pubmed \cite{sen2008collective} datasets. \cora and \pubmed are citation datasets where each
research article is a node, and there exists an edge between two
articles if one article cites the other. Each node has a label that shows
the article category. In \cora, the features of each node are represented by a 0/1-valued word vector, which indicates the word's presence or absence from the article's abstract. In \pubmed, the node features
are represented by the TF/IDF weighted word vector of the unique
words in the dictionary.

\chameleon is a Wikipedia network dataset. Nodes represent articles from the English Wikipedia, edges reflect mutual links between them. Node features indicate the presence of particular nouns in the articles. The nodes were classified into 5 classes in terms of their average monthly traffic.
All datasets are employed for the task of multi-class node classification. 

The statistics of the datasets are provided in 
Table \ref{tab:data-stat}. Here, we compute average label homophily which measures how often adjacent nodes share the same label.
Specifically, for each node $v$ with degree at least $1$ we compute its label homophily as the fraction of its 1-hop neighbors which has the same label as $v$. We then compute average label homophily of the dataset (presented under \textsc{Homophily} in Table \ref{tab:data-stat}) as the average over label homophily of all non-isolated nodes. 
High homophily corresponds to connected nodes tending to share the same labels, as observed in datasets such as \cora and \pubmed, whereas in heterophilic datasets such as \chameleon, connected nodes tend to have dissimilar labels.

\begin{table}
\caption{Dataset statistics. $|\mathcal{V}|$ and $|\mathcal{E}|$ denote the number of nodes and (undirected) edges respectively, $|\mathcal{C}|$ is the number of classes, $d$ is the dimension of the feature vector. Under \textsc{Homophily} we present the average of label homophily of all  non-isolated nodes in the dataset.}
\label{tab:data-stat}
\begin{tabular}{lccccc}
\toprule
 &$|\mathcal{V}|$ & $|\mathcal{E}|$ & $|\mathcal{C}|$ & $d$ & \textsc{Homophily}\\\midrule
\cora & $2,708$ & $5278$ & $7$ & $1,433$& $0.8252$ \\
\chameleon & $2,277$ & $31,371$ & $5$ & $2,325$& $0.2471$\\
\pubmed &$19,717$ & $44,324$ & $3$ & $500$ & $0.7924$\\
\bottomrule
\end{tabular}
\end{table}

\subsection{Settings for Generating Training Splits.} \label{para:settings} We sampled 5 train graphs  each using snowball sampling and random sampling. For all datasets for all splits we fix the number of nodes in the train set to be 10\%  and 50\% of the total nodes. In addition to the usual 50\% split size adopted in evaluations in existing works we chose to include the 10\% split size to simulate realistic settings where the number of member data points (training nodes) is typically much smaller than the overall population.
For snowball sampling we build the initial node set $\mathcal{V}^{(0)}$ by selecting 10 nodes from each class randomly. We fix $k=3$ (number of maximum neighbors to be sampled from each selected node) for all datasets.

In Table \ref{tab:train-degree-10-50} we provide mean of  node degree over the train splits and the corresponding standard deviation as well as the average degree of the original graph.

\input{tables/avgdegree}
\section{Influence of Sampling Strategy and Edge Access}
\label{sec:samplingInf}
Before quantifying their effect on membership advantage (Sec.~\ref{sec:gs_ma}), we \emph{motivate and justify our design choice} to vary two graph–specific factors \emph{training-graph sampling} and \emph{inference-time edge access} by systematically characterizing their impact on (i) the properties of sampled nodes, (ii) the distribution of model outputs, and (iii) distribution of true class confidence scores.

\subsection{Characteristics of Sampled Nodes}
\label{sec:sampled_nodes}
From Table \ref{tab:train-degree-10-50} we observe that snowball sampling leads to higher connectivity, except in the case of chameleon, where random samples have higher edge density. The reason here is that under snowball sampling, we impose (in our experiments) a cap on the maximum number of sampled neighbors, which truncates the neighborhoods of high-degree nodes. This limits the resulting edge density. In contrast, random node sampling does not impose such a restriction, so the average degree of the sampled subgraph can more directly reflect the high average degree of the original graph. In general, as we argue below snowball sampling biases the sample toward high-degree nodes, which affects both how representative the sample is and how GNN models behave.

With uniform random sampling, as studied in the current work, each node in the original graph is included in the sample independently with equal probability $p$. This implies that high-degree and low-degree nodes are sampled with equal likelihood. However, in the induced subgraph formed by the sampled nodes, high-degree nodes tend to appear with lower observed degrees. 
 Specifically, a node of degree $\ell$ in the original graph will retain only those edges that connect to other sampled nodes. The probability that exactly $k$ of its $\ell$ neighbors are also sampled follows a binomial distribution with parameters $\ell$ and $p$. Let $P_\ell$ denote the probability that a node has degree $\ell$ in the original graph, and let $q_k$ denote the probability that a node has degree $k$ in the sampled graph. Then, the degree distribution in the sampled graph is given by:
\begin{align}
    q_k = \sum_{\ell=k}^{\infty} \binom{\ell}{k} p^k (1 - p)^{\ell - k} P_\ell,
\end{align}
This expression accounts for all nodes of degree $\ell \geq k$ in the original graph and computes the probability that exactly $k$ of their neighbors are also sampled. As $p$ decreases, the probability of retaining a high fraction of neighbors drops, leading to a shift in the degree distribution toward lower values.

In contrast, snowball sampling where sampling proceeds by recursively including neighbors of already sampled nodes tends to favor high-degree nodes. This is because such nodes have more opportunities to be reached during the expansion process, and once included, are more likely to bring in additional neighbors. 
To mitigate the rapid expansion caused by high-degree nodes, which, in snowball sampling, can otherwise concentrate the sampled subgraph in very small regions of the graph—we cap the number of neighbors sampled from each node. This constraint also helps control the average degree of the resulting subgraph, keeping it close to a predefined threshold. Nevertheless, the \textbf{probability of selecting a node in the training set increases with its degree} under snowball sampling, whereas it remains uniform under random sampling. An adversary who knows the sampling mechanism can exploit this asymmetry by assigning higher prior membership probabilities to high-degree nodes, even before observing the model output. Moreover, snowball-sampled training graphs tend to concentrate on specific regions of the graph determined by the reachability of the initial seed nodes. This restricted coverage of the graph, as also observed in our experiments, increases the gap between the model’s behavior on member and non-member nodes.


\subsection{Node Neighborhoods and Output Distributions}
\label{sec:node_neighborhoods}
In addition to the influence of the edge structure during training, the output distributions of a GNN for the same set of query nodes can vary at inference time, even when the model parameters are fixed. This variation arises from the inherent \textit{message-passing} or \textit{aggregation} operation in GNNs (cf. Eq.~\eqref{eq:aggre}), which aggregates features from the node's $L$-hop neighborhood. We consider three types of graph access at inference time:

\begin{itemize}[leftmargin=*]
\item \orig: Only the original train edges and test edges are used; train and test sets are edge-disjoint.
\item \alledges: All edges from the underlying graph—including edges between train and test nodes—are used during inference.
\item \nograph: The graph structure is not used at all during inference.
\end{itemize}

Let $\Phi$ be a fixed trained GNN model, and let $v$ be a query node. Let $\mathcal{N}^L_{\mathcal{G}_1}(v)$ and $\mathcal{N}^L_{\mathcal{G}_2}(v)$ denote the $L$-hop neighborhoods of node $v$ in two graphs $\mathcal{G}_1$ and $\mathcal{G}_2$ that differ in edge structure (e.g., different sampling strategies or inference-time neighborhood truncation). Then, the model's output for $v$ under these two neighborhoods is given by:
\[
\mathbf{y}_v^{(1)} := \Phi(v; \mathcal{N}^L_{\mathcal{G}_1}(v)) \quad \text{and} \quad 
\mathbf{y}_v^{(2)} := \Phi(v; \mathcal{N}^L_{\mathcal{G}_2}(v)).
\]
We compute the KL-divergence between the two prediction distributions to capture the discrepancy in the model's output for the same node under different neighborhood contexts:
\[
\mathcal{D}(v; \mathcal{G}_1, \mathcal{G}_2) := D_{KL}(\mathbf{y}_v^{(1)}|| \mathbf{y}_v^{(2)}).
\]
\begin{figure}
    \centering
    \includegraphics[width=0.8\linewidth]{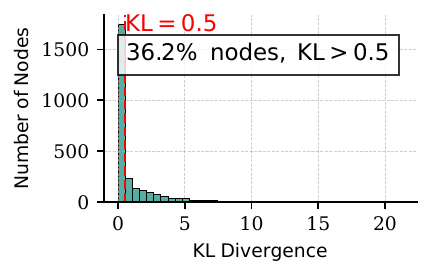}
    \caption{Distribution of KL divergence among posterior distribution of nodes on \cora dataset comparing the cases when all edges and none of the edges were used during inference. The model (\gcn) was trained on a {\color{blue}{snowball}} sampled split with 50\% of the nodes in the train set.}
    \label{fig:cora_kl_snowball}
\end{figure}
\begin{figure}
    \centering
    \includegraphics[width=0.8\linewidth]{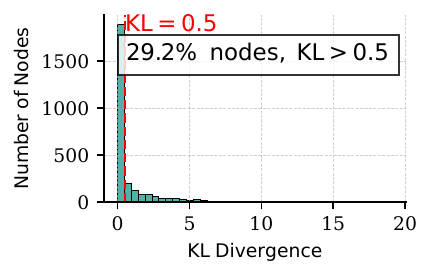}
    \caption{Distribution of KL divergence among posterior distribution of nodes on \cora dataset comparing the cases when all edges and none of the edges were used during inference. The model (\gcn) was trained on a {\color{blue}{randomly}} sampled split with 50\% of the nodes in the train set.}
    \label{fig:cora_kl_random}
\end{figure}
This divergence can be non-zero even for nodes in the training set, highlighting that inference-time neighborhood construction can directly and significantly affect the resulting node representations and, consequently, the model’s predictions. An extreme illustration of this occurs when all neighbors are included during inference versus when all edges are ignored resulting in entirely different aggregation contexts and hence different predictions. In Figures~\ref{fig:cora_kl_snowball} and~\ref{fig:cora_kl_random}, we show how the output posterior distribution, that is, the predicted probability vector, changes for the same nodes and the same model when only the neighborhood structure is varied. The large differences in model posteriors for the same nodes imply that attack success rates can vary substantially for those nodes depending on the adversary’s knowledge of the edge structure.

\subsection{Distribution of True Class Confidence}

For each node, we map the prediction confidence \(p\in(0,1)\) to
the \emph{log-odds} (logit) 
\[
z \;=\; \operatorname{logit}(p) \;=\; \log\!\left(\frac{p}{1-p}\right).
\]
This places probabilities on an unbounded, symmetric scale
(\(z=0\) at \(p=0.5\); \(z>0\) iff \(p>0.5\)), which spreads out extreme
values near \(0/1\) and facilitates distributional comparisons \cite{carlini2022membership}.
To avoid infinities at \(p\in\{0,1\}\), we clamp
\(p \leftarrow \min\!\big(\max(p,\varepsilon),\,1-\varepsilon\big)\)
with \(\varepsilon=10^{-6}\) before applying the transform.
(The inverse mapping is the sigmoid \(p=\sigma(z)=1/(1+e^{-z})\).) 

In Figure \ref{fig:classconfidence} we plot the distribution for logit transformed true class confidence scores for \cora with \sage under different edge access settings. Note that the trained model as well as the membership and non-membership status of nodes is fixed according to the original split. Only edge access is changed during inference to obtain prediction scores. We observe that the scores provide greater distinguishability when edges from the original split are included at inference compared to the other settings. 

\begin{figure*}
   \centering\includegraphics[width=0.9\linewidth]{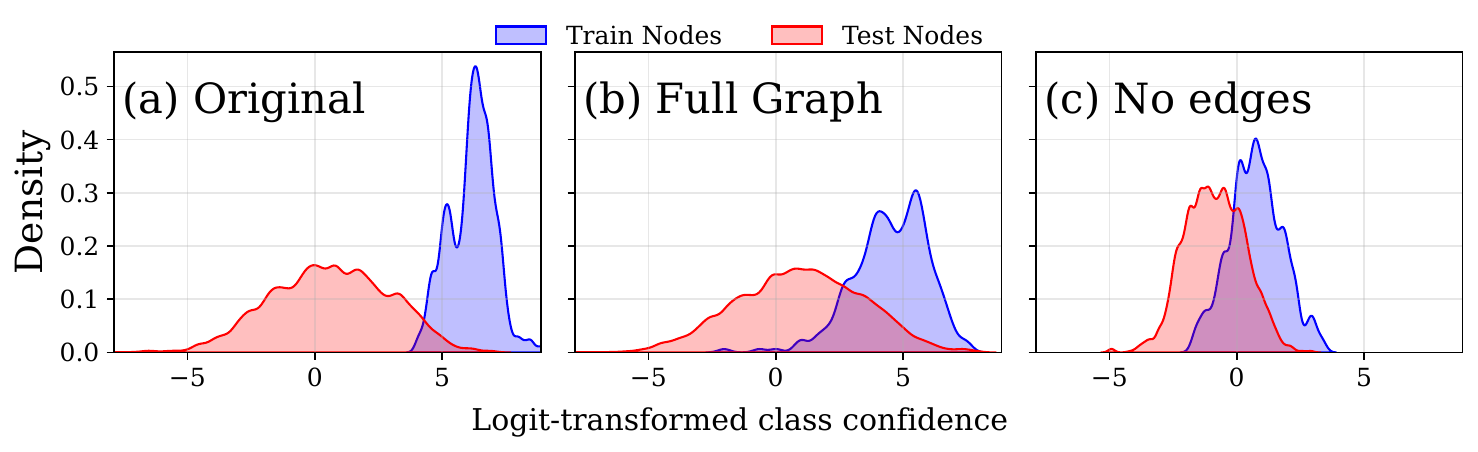}
  \caption{Effect of edge structure on train--test separability of class confidence
(log-odds of the true-class prediction probability) for a \sage model. The model was trained over \cora dataset with 10\% of nodes sampled for training using \textbf{snowball} sampling. 
We compare three inference graphs: (a) \emph{Original} (train/test edge sets disjoint),
(b) \emph{Full graph} (all edges available), and (c) \emph{No edges}.
Train (blue) and test (red) distributions are most separated in (a),
indicating stronger distinguishability; in (b) and (c) they overlap more,
indicating weaker distinguishability. }
\label{fig:classconfidence}
\end{figure*}

\section{Graph Structure and Membership Advantage}
\label{sec:gs_ma}
After presenting qualitative discussion and preliminary evidence in favour of our argument to study the role of structure we now present large-scale quantitative results on how edge structure impacts membership advantage in GNNs. We begin by analyzing the effect of \emph{training-graph construction} and \emph{inference-time edge access}, on predictive performance (Section~\ref{sec:gs-mp}), and then study how the resulting train–test performance gap correlates with membership advantage (Section~\ref{sec:pgma}).

\subsection{Experimental Setup} 
\label{sec:exp}
For each dataset we perform our experiments on 5 training splits constructed using snowball sampling and random sampling as described in Section \ref{sec:sampling}.  The attack model is a 2-layer MLP trained on an $r$ fraction of member and non-member data points with their true labels. By varying \(r \in \{0.1, 0.2, 0.5, 0.8\}\), we construct attack models with different levels of adversarial knowledge, allowing us to study whether the structural effects observed in our investigation remain consistent under different attacker strengths. For each of the training-test data split of the target model, we generate 3 random splits for each of the attack models. The attack results corresponding to each $r$ are then summarized over a total of 15 runs for each dataset. We employ \gcn \cite{kipf2016semi}, \gat \cite{velivckovic2017graph} and \sage \cite{hamilton2017inductive} as the GNN models. All GNN models are evaluated with exactly the same data splits. Hyperparameter details of our implementation are provided in Appendix \ref{sec:imp}.

 We compute the performance gap as the percentage decrease in test accuracy relative to the training accuracy, i.e.,
$$ \Delta_{gen} = {ACC_{train}-ACC_{test} \over ACC_{train} }\times 100.$$ We employ performance gap between the train and test sets to approximate generalizaton gap. We consider three types of graph access at inference time:

\begin{itemize}[leftmargin=*]
\item \orig: Only the original train edges and test edges are used; train and test sets are edge-disjoint.
\item \alledges: All edges from the underlying graph—including edges between train and test nodes—are used during inference.
\item \nograph: All edge connections, that is, the graph structure is completely discarded during inference.
\end{itemize}

\subsection{How do Graph Sampling and Edge Access Affect Performance Gap?}
\label{sec:gs-mp}
\begin{figure}
    \centering
    \includegraphics[width=1\linewidth]{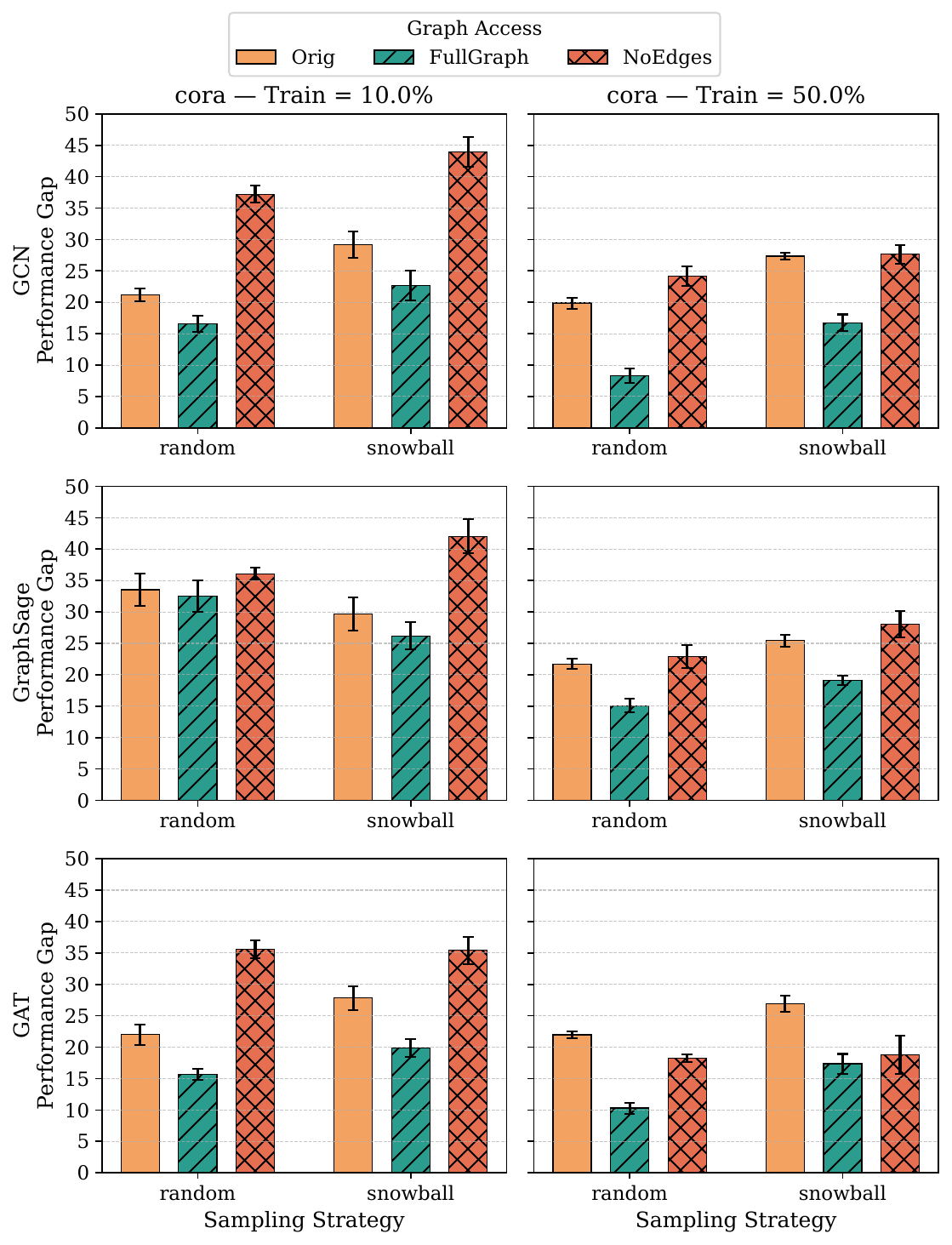}
    \caption{Performance gap (in \%) under different edge access settings for \cora. The left plot corresponds to a training size of 10\% of nodes, while the right plot corresponds to a training size of 50\% of nodes.}
    \label{fig:corapg}
\end{figure}
\begin{figure}
    \centering
    \includegraphics[width=1\linewidth]{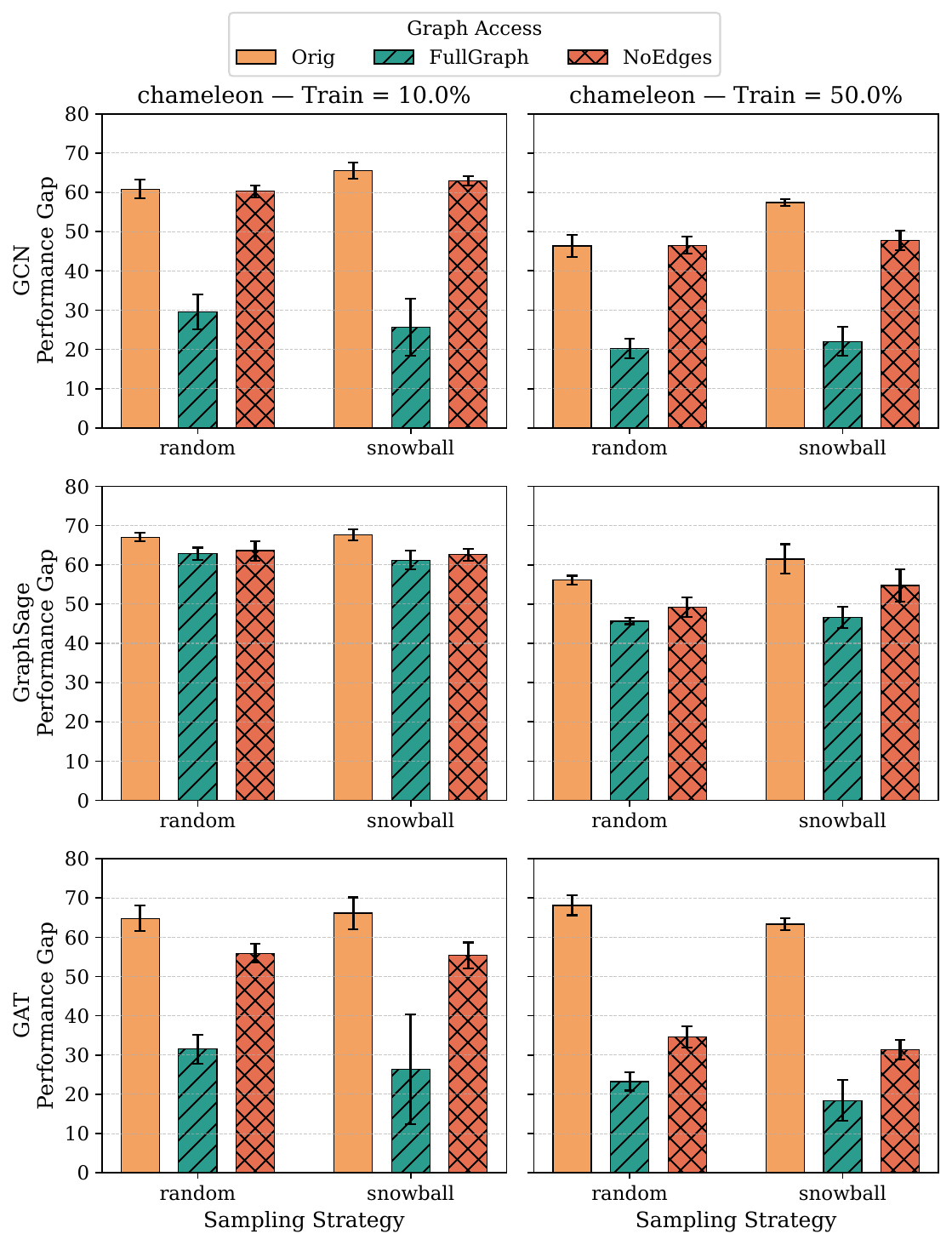}
    \caption{Performance gap (in \%) under different edge access settings for \chameleon. Left plot corresponds to a training size of 10\% of nodes, while the right plot corresponds to a training size of 50\% of nodes.}
    \label{fig:chameleonpg}
\end{figure}
\begin{figure}[h!]
    \centering
    \includegraphics[width=1\linewidth]{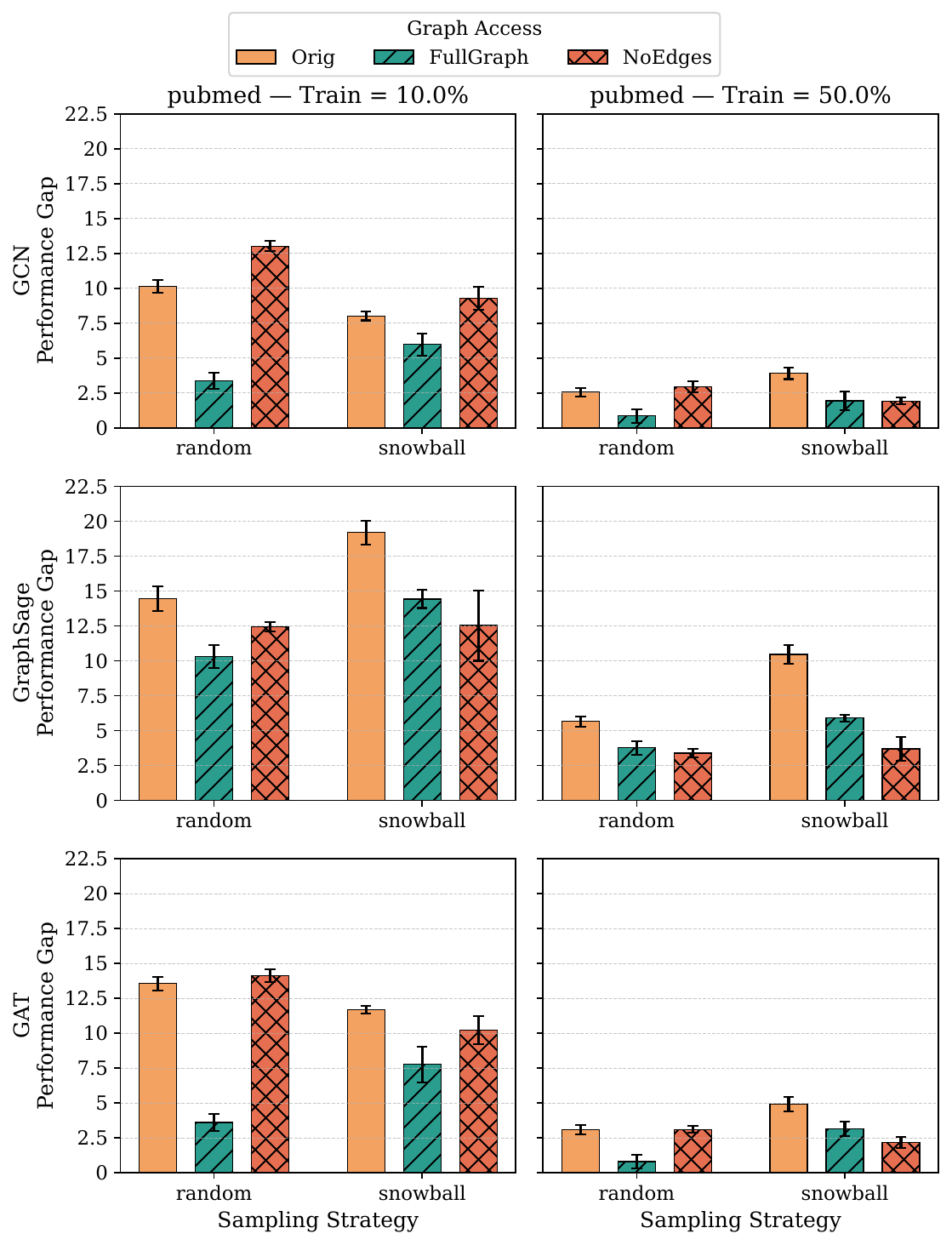}
    \caption{Performance gap (in \%) under different edge access settings for \pubmed. The left plot corresponds to a training size of 10\% of nodes, while the right plot corresponds to a training size of 50\% of nodes.}
    \label{fig:pubmedpg}
\end{figure}

In figures \ref{fig:corapg}, \ref{fig:chameleonpg}, \ref{fig:pubmedpg} we compare the performance gap of the three GNNs under two train-graph construction strategies and the different edge access levels for \cora, \chameleon and \pubmed datasets. Detailed train and test accuracy scores for all datasets are presented in Tables
\ref{tab:performance_cora_train10}--\ref{tab:performance_cora_train50} (\cora), \ref{tab:performance_chameleon_train10}--\ref{tab:performance_chameleon_train50} (\chameleon), and
\ref{tab:performance_pubmed_train10}--\ref{tab:performance_pubmed_train50} (\pubmed) in Appendix \ref{sec:testtrainaccu}.

\subsubsection{Effect of Sampling Type}
When comparing different sampling strategies for constructing the training graph, we would expect higher performance gap when train graphs are snowball sampled. The rational behind this is that with random sampling, despite disrupting the graph structure, we obtain a i.i.d. node distribution. In contrast, snowball sampling introduces bias by expanding from a seed set, increasing the likelihood of sampling nodes that are structurally similar or closely connected. While the expected trend is observed over all datasets there are a few exceptions when $10\%$ of nodes were used for training. Overall the observed performance gaps are highest for low homophilic dataset \chameleon and lowest for \pubmed. 

\subsubsection{Effect of Edge Access} The performance gap varies depending on the level of graph access. The gap is usually smallest when all edges are used highlighting the benefit of including inter-train-test edges in improving test-time predictions. For \chameleon, although the performance gap is lower in the \alledges setting, both train and test accuracies decrease for \gcn and \gat (with \sage as the main exception). This is expected in a low-homophily graph, where neighboring nodes are less likely to share labels: adding more edges can therefore hurt the message-passing behavior of \gcn and \gat, which implicitly smooth representations across adjacent nodes. The case of \sage is particularly interesting, as the accuracy on training nodes changes little across edge-access settings. We hypothesize that this is due to its aggregation mechanism, which concatenates the ego-node representation with the aggregated neighbor representation rather than averaging them together. As we show later, this distinction also has important consequences for privacy leakage when we examine membership advantage across these settings. We also see some exceptions in \pubmed where in \alledges setting the performance gap is sometimes higher than that of \nograph setting. At the moment we do not have sound explanation for the phenomena. What is important, however, is that different assumptions about edge access at inference time do induce different levels of performance gap, which in turn can lead to different levels of privacy leakage.

\subsubsection{Effect of Sampling Size
} When the number of training nodes is increased from 10\% to 50\% a decrease in performance gap as expected is observed (with minor exceptions) for all models at different graph access levels and train sampling strategies.

\begin{figure*}[h!]
    \centering
    \includegraphics[width=0.9\linewidth]{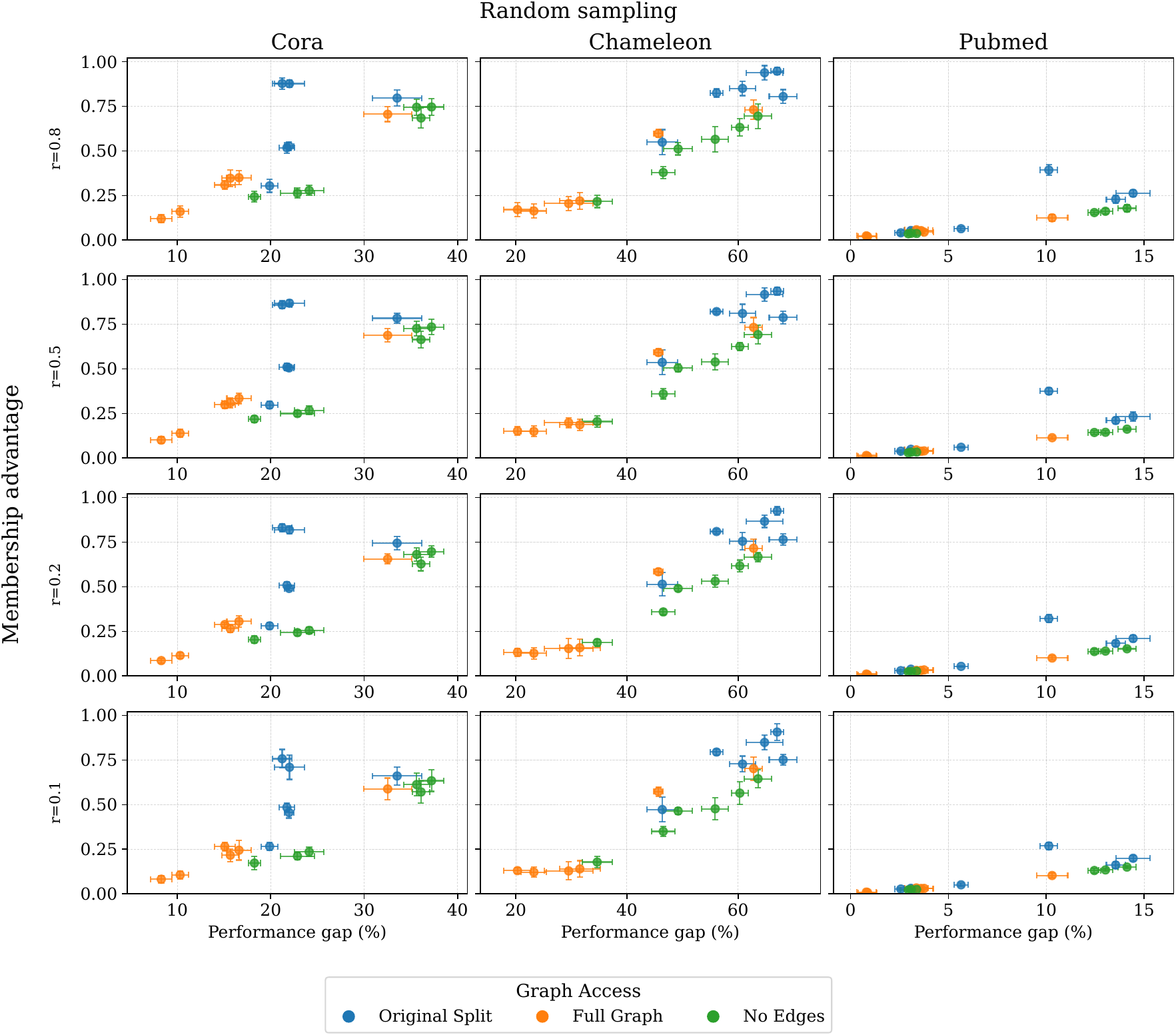}
    \caption{Scatter plot of performance gap versus membership advantage under {\color{blue}{random sampling}} shown for different attack-model training fractions $r\in\{0.1,0.2,0.5,0.8 \}$. Each row corresponds to one value of $r$, and points are grouped by edge-access setting (Original Split, Full Graph, and No Edges). Each point shows the mean membership advantage at a given mean performance gap, while horizontal and vertical error bars denote one standard deviation in performance gap and membership advantage, respectively.}
    \label{fig:perfmarandom}
\end{figure*}

\begin{figure*}
    \centering
    \includegraphics[width=0.9\linewidth]{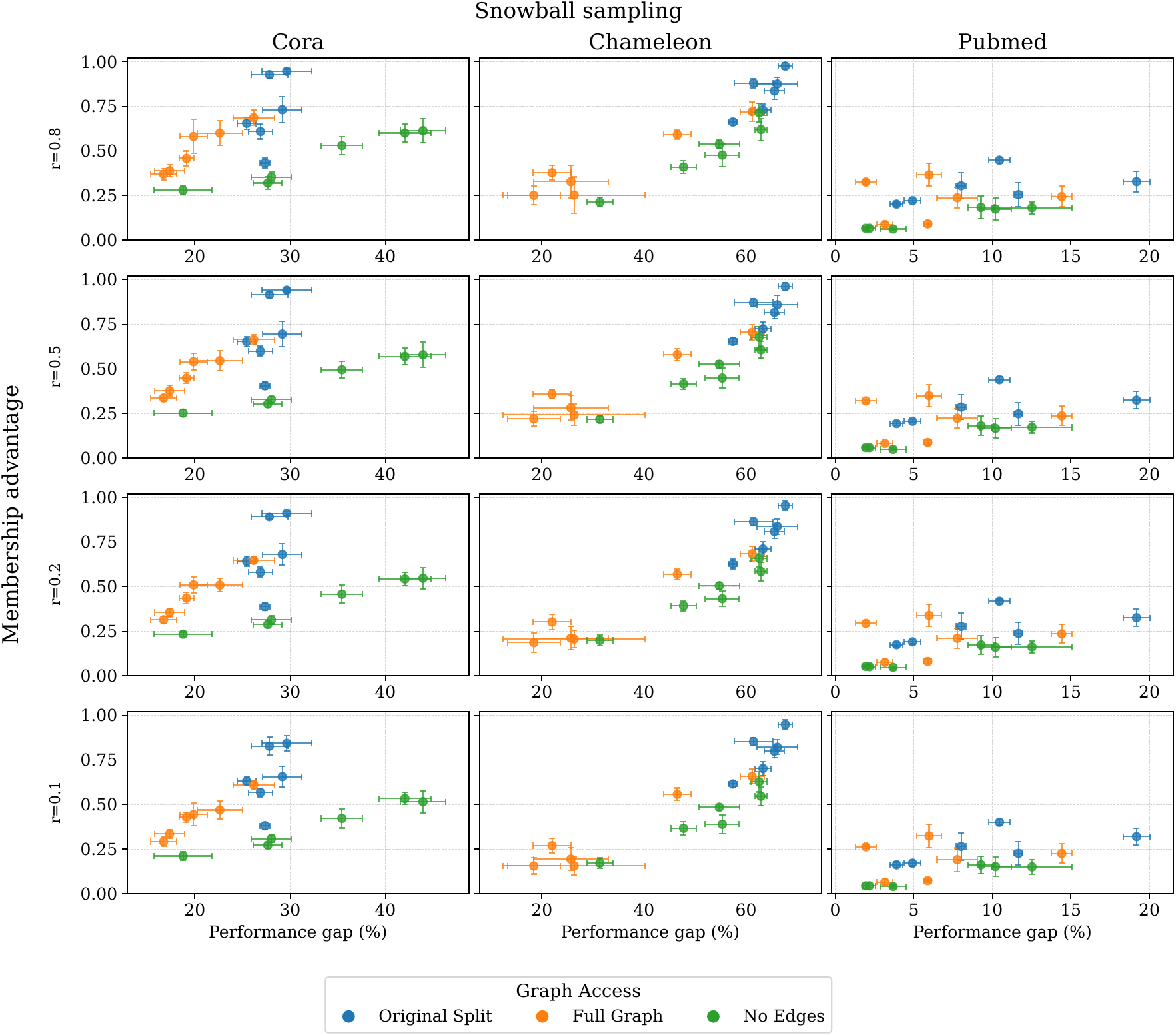}
\caption{Scatter plot of performance gap versus membership advantage under {\color{blue}{snowball sampling}} shown for different attack-model training fractions $r\in\{0.1,0.2,0.5,0.8 \}$. Points and error bars have the same meaning as in the corresponding plot for random sampling.}
    \label{fig:perfmasnowball}
\end{figure*}

\subsection{How does Performance Gap Translate to Membership Advantage?}
\label{sec:pgma}
Figures~\ref{fig:perfmarandom} and \ref{fig:perfmasnowball} show scatter plots of performance gap versus membership advantage, aggregated over all models and settings, for the two sampling techniques. Each row corresponds to a different attacker strength parameter \(r \in \{0.1, 0.2, 0.5, 0.8\}\), representing the fraction of member and non-member nodes whose true membership labels are revealed to the attack model during training. Points in these plots are grouped by edge-access setting (Original Split, Full Graph, and No Edges). Each point shows the mean membership advantage at a given mean performance gap, while horizontal and vertical error bars denote one standard deviation in performance gap and membership advantage, respectively. We provide detailed performance-gap and membership-advantage scores for all datasets in Tables
\ref{tab:gap_adv_cora_train10_r10}--\ref{tab:gap_adv_cora_train50_r80} (\cora),
\ref{tab:gap_adv_chameleon_train10_r10}--\ref{tab:gap_adv_chameleon_train50_r80} (\chameleon), and
\ref{tab:gap_adv_pubmed_train10_r10}--\ref{tab:gap_adv_pubmed_train50_r80} (\pubmed) in Appendix \ref{sec:testtrainaccu}.

We observe that across all values of \(r\), the relative ordering of the different graph-access settings and the overall relationship between performance gap and membership advantage remain largely stable, with only isolated exceptions. In addition, membership advantage tends to increase with \(r\), especially when viewed in aggregate across settings. We therefore provide the following analysis without referring to any specific value of \(r\).

First, from the scatter plots, we observe that in terms of membership advantage the original split setting dominates. Specifically, we observe that even when No Edges setting  exhibits a comparable or even larger performance gap than the Original Split setting, it nevertheless leads to lower membership leakage.

Second, from the detailed scores we observe that in most cases, particularly for \gcn and \gat, moving from the Original Split to the Full Graph leads to a sharp drop in membership advantage. This trend is observed on both \cora and \chameleon, suggesting that access to the full graph often makes member and non-member nodes less distinguishable. The train and test accuracies provide further insight into this effect. On \cora, using all graph edges increases test accuracy, while on \chameleon, using all graph edges causes a substantial decrease in train accuracy and, in most cases, a small increase in test accuracy,  thereby reducing the performance gap. This is consistent with the fact that \chameleon is a low-homophily graph, where adding more edges can degrade performance. Even so, message passing still creates a mixing effect that blurs the distinction between member and non-member nodes. 

Third, the differences in membership advantage across different edge-access settings are often smaller for \sage. We hypothesize that this stems from \textsc{GraphSage}'s  separate treatment of self-node and neighborhood information, which can preserve more node-specific signal and make \sage\ less sensitive to biased neighbor sets than \gcn\ or \gat, which more directly mix node and neighbor messages.

Fourth, we discuss the special case of \pubmed. Across both train-set sizes and sampling strategies, train and test accuracies remain relatively high, and the resulting membership advantage is comparatively low. This is broadly consistent with prior observations in the literature regarding the relative difficulty of attacking this dataset. A notable departure from the pattern seen in \cora\ and \chameleon\ concerns the relationship between performance gap and membership advantage in the \alledges setting. On \pubmed, using the full graph does not suppress leakage as consistently as it does on \cora. In particular, for \gcn\ under Snowball sampling, \alledges setting yields membership advantage that is comparable to, and in some cases higher than, that of the \orig setting, even though the corresponding performance gap is lower.

Last, when comparing snowball and random sampling, we generally observe higher membership advantage under snowball sampling. This effect is most pronounced on \pubmed, where, despite the overall membership advantage being relatively low, snowball sampling still consistently exhibits higher membership advantage compared to random sampling. A likely explanation is the coverage bias introduced by snowball sampling, which concentrates the sampled training graph within specific structural regions and repeatedly exposes the model to similar neighborhood patterns, thereby increasing behavioral differences between member and non-member nodes.

One important subtlety we observe is that for the homophilic graphs \cora and \pubmed, the \alledges setting can exhibit higher membership advantage than the \nograph setting, even when the former achieves a lower performance gap. This behavior contrasts with the random sampling setting, where \alledges often leads to the lowest membership advantage.



\textbf{Summary.} Overall our experimental results point to the fact that the edge structure is a crucial variable for understanding and quantifying privacy leakage. Alone performance or generalization gap does not suffice to explain the information leakage in graph based models. 
For realistic evaluation of privacy leakage in GNNs we need to devise new and realistic assumptions on the available graph structure to the adversary. To develop more intelligent attacks new edge manipulation strategies would need to be developed that can help attacker to behave differently on member and non-member query nodes.

%% file: tables/avgdegree.tex
\begin{table}[ht]
\centering
\setlength{\tabcolsep}{3.5pt}
\caption{Average degree of \emph{train graphs} when $10\%$ or $50\%$ of nodes are sampled using \textsc{snowball} ($k{=}3$) or \textsc{random} and the underlying full graph without any sampling.  Values are mean $\pm$ std over $5$ splits.}
\label{tab:train-degree-10-50}
\resizebox{\columnwidth}{!}{%
\begin{tabular}{@{}>{\centering\arraybackslash}p{8mm} l c c c@{}}
\toprule
{} & \textbf{Sampling Type} & \makecell{\cora} & \makecell{\textsc{Chameleon}} & \makecell{\pubmed} \\
\midrule
\multirow{2}{*}{\rotatebox[origin=c]{90}{10\%}}
& \textsc{Snowball}   & $1.68 \pm 0.046$ & $1.93 \pm 0.052$ & $2.55 \pm 0.058$ \\
& \textsc{Random}     & $0.32 \pm 0.063$ & $2.79 \pm 0.495$  & $0.41 \pm 0.034$ \\
\midrule
\multirow{2}{*}{\rotatebox[origin=c]{90}{50\%}}
& \textsc{Snowball}   & $3.18 \pm 0.051$ & $4.91 \pm 0.143$  & $3.55 \pm 0.013$ \\
& \textsc{Random}     & $2.00 \pm 0.111$ & $13.52 \pm 0.741$ & $2.21 \pm 0.029$ \\
\midrule
\multicolumn{1}{c}{} & \textsc{Full Graph} & $3.90 \pm 0.00$ & $27.55 \pm 0.00$ & $4.496 \pm 0.00$ \\
\bottomrule
\end{tabular}%
}
\end{table}

%% file: relatedwork.tex
\section{Related Work}
\label{sec:related}

Membership inference (MI) is a fundamental privacy risk for machine learning: black\mbox{-}box and white\mbox{-}box attacks exploit systematic differences in a model’s behavior on training (member) vs.\ held\mbox{-}out (non\mbox{-}member) data \cite{shokri2017membership,carlini2022membership,yeom2018privacy,salem2018ml,hu2022membership,humphries2023investigating}. 
Recent work extends MI to graph neural networks (GNNs) by defining the \emph{instance} at the node, edge, or graph level and tailoring the attack surface accordingly \cite{wang2021membership,wu2021adapting,olatunji2021membership,he2021stealing,wang2024subgraph,dai2022comprehensive}. 
Typical node\mbox{-}level attacks query a trained GNN with a target node (and optionally its neighborhood), using confidence scores, losses, or shadow\mbox{-}model posteriors as membership signals; threat models vary in (i) access to outputs/parameters and (ii) the relational context included with the query (features only vs.\ features$+$edges). 

Despite this progress, the use of  \emph{relational structure} that distinguishes GNNs from i.i.d.\ models has not been examined systematically as a driver of MI risk. In particular, prior attacks typically treat the graph as a fixed backdrop and do not analyze how training graph construction  and inference time edge access modulates attack success. 


To defend against information leaks via membership inference with theoretical gurantees, \emph{differential privacy} (DP) \cite{dwork2006calibrating} has emerged as the most widely used formalism: it guarantees that the presence or absence of any single data point has only a limited effect on a mechanism’s output, thereby constraining an adversary’s ability to infer membership. 
Existing methods for privacy preserving training of GNNs in centralised setting can be grouped into three main classes: (i) knowledge distillation based \cite{papernot2018scalable} in which the model trained on private data is never released but is used to train a public model under a weak supervision setting \cite{olatunji2023releasing}, (ii) gradient perturbation techniques in which typically DP-SGD \cite{abadi16} is extended for the case of graphs \cite{xiang2024preserving}, and  (iii) aggregation perturbation where the key idea is to add noise to the aggregate information obtained from the GNN neighborhood aggregation step \cite{sajadmanesh2022gap,sajadmanesh2021locally}.

In i.i.d.\ domains, membership–inference (MI) attacks are widely used to \emph{audit} differentially private (DP) learning by providing empirical lower bounds on leakage \cite{yeom2018privacy,humphries2023investigating,jayaraman2019evaluating,jagielski2020auditing}. While it may seem that these auditing algorithms carry over directly to graphs, we showed that this intuition can be misleading. We made the assumptions explicit and provided \emph{theoretical} support for how train graph construction strategies affect the validity of these audits.


%% file: discussion.tex
\section{Discussion and Key Insights}
We investigated node-level membership inference (MI) against GNNs in the inductive setting, formalizing the problem through a joint distribution over node--neighborhood tuples and an explicit membership inference experiment. Our analysis is grounded in a key distinction between graph learning and classical non-connected data settings: in GNNs, understanding privacy risk fundamentally requires understanding how the training graph is constructed, what structural properties it exhibits, and how graph structure is utilized during training and inference. Our theoretical and empirical results lead to the following key insights, which we hope will motivate future work in this direction.

\subsection{Generalization Gap in GNNs}
In machine learning, the generalization gap typically refers to the difference between empirical performance on the training set and expected performance on unseen samples drawn from the same underlying data distribution. A large gap is often associated with overfitting and increased memorization of training samples.

In graph-based node-level learning, however, the relationship between generalization gap and membership inference risk becomes more subtle because node instances are not processed independently. In GNNs, predictions depend not only on node features, but also on the neighborhood structure available during training and inference. Consequently, the generalization gap may depend not only on which nodes belong to the training set and the underlying node feature distribution, but also on which edges and neighborhood connections are exposed during training and inference.

Our empirical analysis confirms this subtlety. We observe that changing only the exposed edge structure can substantially alter the performance gap, while its relationship to membership advantage does not always follow the expected trend. In particular, when no edges are exposed during inference, the performance gap may increase relative to the Original split setting, whereas the membership advantage simultaneously decreases. Such observations suggest that classical notions of overfitting and memorization may need to be revisited in graph learning settings, since differences in model behavior can arise not only from learned parameters, but also from the structural context available during inference.

\subsection{Consequences for Differentially Private Models.}
We showed that train and test instances are not statistically exchangeable in practical graph learning settings (c.f. Theorem~\ref{thm:non_exchangeability}). As a result, the empirical membership advantage computed using membership inference attacks, including the experiment formalized in this work, cannot in general be interpreted as a lower bound on the differential privacy (DP) budget in graph settings.

To understand the above in more practical terms, consider first a strong adversary with access to the sampling strategy and \(n-1\) data points. In such cases, the remaining node may become distinguishable purely because of structural constraints induced by the graph construction process. In particular, as shown in the proof of Theorem \ref{thm:non_exchangeability}, certain node--neighborhood configurations may become impossible under specific sampling strategies and observed graph structure, allowing membership to be inferred independently of the learned model parameters, even when the model itself is trained using DP.

Furthermore, even for weaker adversaries without access to the remaining \(n-1\) data points, graph-dependent sampling strategies can still influence the learning process through uneven graph coverage and under-representation of particular structural regions or node types. Such coverage bias can affect model generalization and increase behavioral differences between member and non-member nodes. Importantly, this distinguishability may arise even when the learning algorithm itself satisfies differential privacy, because part of the leakage originates from structural biases introduced by the graph construction and sampling process. Consequently, the observed membership advantage may not be fully explained or bounded by the DP parameters of the trained model alone.
We therefore advocate the need for new privacy notions in graph machine learning that explicitly account for the training graph construction strategy.
\subsection{Modeling Privacy Risk }
As final remarks, our analysis suggests that both privacy attacks and defenses in graph learning require more realistic threat formulations. An important open question is what can meaningfully be protected once the graph construction or sampling mechanism itself is even partially known to the adversary.  Moreover, since node properties themselves are influenced by the surrounding graph structure, different nodes cannot always be treated as equivalent from a privacy perspective. Nodes occupying different structural roles may inherently exhibit different levels of membership risk, requiring more fine-grained analysis of privacy leakage in graph learning settings.




%% file: appendix.tex
\appendix
\section{Aggregation Operations in GNNs}
\label{sec:aggre}
We focus on three key aggregation operations in GNNs, which serve as the foundation for more complex aggregation methods. The first aggregation operation is based on the Graph Convolutional Network (\gcn) model proposed by \cite{kipf2016semi}. In this model, each node's representation is updated at each layer by computing a weighted average of its own features and those of its neighbors. 
Let $\mathbf{d}$ be the graph degree vector obtained after adding self loops to all nodes. The $ith$ element $\mathbf{d}_i$ denotes the degree of node $i$. The aggregation operation in \gcn is then given as 
\begin{equation}\label{eq:conv}
\mathbf{z}_{i}^{(\ell)} \leftarrow \mathbf{W}^{(\ell)}\sum_{j\in \mathcal{N}(i)\cup i} \frac{1}{\sqrt{\mathbf{d}_{i}\mathbf{d}_{j}}} \mathbf{x}_{j}^{(\ell-1)}.
\end{equation}
Here $\mathbf{W}^{(\ell)}$ is the projection matrix corresponding to layer $\ell$.
Unlike \gcn, \sage \cite{hamilton2017inductive} explicitly differentiates between the representation of the query node (referred to as the ego node) and that of its neighbors. In \sage, aggregation is performed by concatenating the ego node’s representation with the aggregated representations of its neighbors. Several methods for neighbor aggregation are proposed; in this work, we adopt the mean aggregation method, which computes the average of the neighbors' representations. After applying the projection operation, the simplified aggregation in \sage can be expressed as follows:
\begin{equation}
\label{eq:sage}
  \mathbf{z}^{(\ell)} = \mathbf{W}^{(\ell)}_1 \mathbf{x}^{(\ell-1)}_i + \mathbf{W}^{(\ell)}_2 \cdot
        {1\over |\mathcal{N}(i)|}\sum_{j \in \mathcal{N}(i)} \mathbf{x}^{(\ell-1)}_j
\end{equation}
In both \gcn and \sage, the transformation step following aggregation is simply the application of a ReLU non-linearity to the aggregated representation.
\gat \cite{velivckovic2017graph} introduces attention weights
over the edges to perform the aggregation operation as follows.
\begin{align}
   \mathbf{z}_i^{(\ell,p)} = 
        \sum_{j \in \mathcal{N}(i) \cup i}
        \alpha^{(p)}_{i,j}\mathbf{W}^{(\ell,p)}\mathbf{x}_{j}^{(\ell-1)}, 
\end{align}
where the attention coefficients $\alpha^{(p)}_{i,j}: (\mathbf{x}_{i},\mathbf{x}_{j})\rightarrow \mathbb{R}$ can be seen as the importance weights for aggregation over edge $(i,j)$ and computed as in \cite{velivckovic2017graph}. In  the transformation step the $P$ representations corresponding to $P$ attention mechanisms are concatenated after a non-linear transformation to obtain a single representation at layer $\ell$.
\begin{equation}
\label{eq:gat}
    \mathbf{x}_{i}^{(\ell)}=||_{p=1}^{P} \operatorname{ReLU}\left( \mathbf{z}_{i}^{(\ell,p)} \mathbf{W}^{(p\ell)} \right).
\end{equation}
In this work we only experiment with the above representative aggregation types and do not account for other differences in these methods. For example, \sage was proposed to improve scalability in GNN training in which neighborhood sampling was used to reduce the size of the computational graph. In our implementation we restrict to only vary the aggregation types and the provided neighborhood of the input query node is completely used.
\section{Hyperparameter Details}
\label{sec:imp}
 All three target models were implemented as two-layer graph neural networks with a hidden dimension of $64$ and were trained for $200$ epochs using the Adam optimizer with learning rate $0.01$ and weight decay $5 \times 10^{-4}$. For \gat, we used $8$ attention heads in the first layer and a single attention head in the output layer. 

The membership inference attack model was implemented as a 2-layer MLP with a hidden layer of dimension $64$ and a single output unit.
 Each of the attack models was trained for $200$ epochs using Adam with learning rate $0.001$ and batch size $32$. Binary cross-entropy loss was applied to sigmoid-transformed outputs. No dropout or additional hyperparameter tuning was used in the current implementation.

\section{Detailed Results}
\label{sec:testtrainaccu}
Detailed train and test accuracy scores for all datasets are presented in Tables
\ref{tab:performance_cora_train10}--\ref{tab:performance_cora_train50} (\cora), \ref{tab:performance_chameleon_train10}--\ref{tab:performance_chameleon_train50} (\chameleon), and
\ref{tab:performance_pubmed_train10}--\ref{tab:performance_pubmed_train50} (\pubmed). We further provide detailed performance-gap and membership-advantage scores for all datasets in Tables
\ref{tab:gap_adv_cora_train10_r10}--\ref{tab:gap_adv_cora_train50_r80} (\cora),
\ref{tab:gap_adv_chameleon_train10_r10}--\ref{tab:gap_adv_chameleon_train50_r80} (\chameleon), and
\ref{tab:gap_adv_pubmed_train10_r10}--\ref{tab:gap_adv_pubmed_train50_r80} (\pubmed).

\input{tables/cora_all_accuracy_tables}
\input{tables/chameleon_all_accuracy_tables}
\input{tables/pubmed_all_accuracy_tables}
\input{tables/cora_all_gap_adv_tables}
\input{tables/chameleon_all_gap_adv_tables}
\input{tables/pubmed_all_gap_adv_tables}




%% file: tables/cora_all_accuracy_tables.tex
\begin{table*}[t]
\centering
\small
\caption{Performance on Cora with 10\% of nodes in the training set of the victim model under random and snowball sampling.}
\label{tab:performance_cora_train10}
\begin{tabular}{llcccccc}
\toprule
\multirow{2}{*}{Model} & \multirow{2}{*}{Metric} & \multicolumn{3}{c}{Random} & \multicolumn{3}{c}{Snowball} \\
\cmidrule(lr){3-5} \cmidrule(lr){6-8}
 &  & Original split & Full graph & No edges & Original split & Full graph & No edges \\
\midrule
\multirow{2}{*}{GCN} & Train Acc & $0.9874 \pm 0.0083$ & $0.9578 \pm 0.0127$ & $0.9896 \pm 0.0064$ & $0.9963 \pm 0.0023$ & $0.9615 \pm 0.0095$ & $0.9459 \pm 0.0093$ \\
 & Test Acc & $0.7779 \pm 0.0136$ & $0.7987 \pm 0.0136$ & $0.6214 \pm 0.0144$ & $0.7057 \pm 0.0210$ & $0.7438 \pm 0.0191$ & $0.5302 \pm 0.0224$ \\
\midrule
\multirow{2}{*}{GraphSAGE} & Train Acc & $1.0000 \pm 0.0000$ & $0.9985 \pm 0.0030$ & $0.9807 \pm 0.0059$ & $1.0000 \pm 0.0000$ & $0.9956 \pm 0.0028$ & $0.9319 \pm 0.0268$ \\
 & Test Acc & $0.6647 \pm 0.0261$ & $0.6738 \pm 0.0261$ & $0.6268 \pm 0.0110$ & $0.7035 \pm 0.0264$ & $0.7349 \pm 0.0227$ & $0.5400 \pm 0.0303$ \\
\midrule
\multirow{2}{*}{GAT} & Train Acc & $0.9867 \pm 0.0093$ & $0.9356 \pm 0.0160$ & $0.9800 \pm 0.0038$ & $0.9911 \pm 0.0055$ & $0.9296 \pm 0.0189$ & $0.8911 \pm 0.0138$ \\
 & Test Acc & $0.7696 \pm 0.0200$ & $0.7889 \pm 0.0164$ & $0.6310 \pm 0.0155$ & $0.7153 \pm 0.0190$ & $0.7450 \pm 0.0213$ & $0.5751 \pm 0.0166$ \\
\bottomrule
\end{tabular}
\end{table*}

\begin{table*}[t]
\centering
\small
\caption{Performance on Cora with 50\% of nodes in the training set of the victim model under random and snowball sampling.}
\label{tab:performance_cora_train50}
\begin{tabular}{llcccccc}
\toprule
\multirow{2}{*}{Model} & \multirow{2}{*}{Metric} & \multicolumn{3}{c}{Random} & \multicolumn{3}{c}{Snowball} \\
\cmidrule(lr){3-5} \cmidrule(lr){6-8}
 &  & Original split & Full graph & No edges & Original split & Full graph & No edges \\
\midrule
\multirow{2}{*}{GCN} & Train Acc & $0.9880 \pm 0.0022$ & $0.9499 \pm 0.0058$ & $0.8954 \pm 0.0050$ & $0.9968 \pm 0.0013$ & $0.9694 \pm 0.0049$ & $0.8471 \pm 0.0044$ \\
 & Test Acc & $0.7916 \pm 0.0088$ & $0.8712 \pm 0.0090$ & $0.6795 \pm 0.0151$ & $0.7241 \pm 0.0054$ & $0.8072 \pm 0.0098$ & $0.6129 \pm 0.0125$ \\
\midrule
\multirow{2}{*}{GraphSAGE} & Train Acc & $1.0000 \pm 0.0000$ & $0.9929 \pm 0.0027$ & $0.8913 \pm 0.0158$ & $1.0000 \pm 0.0000$ & $0.9985 \pm 0.0014$ & $0.8781 \pm 0.0097$ \\
 & Test Acc & $0.7827 \pm 0.0079$ & $0.8431 \pm 0.0111$ & $0.6873 \pm 0.0090$ & $0.7456 \pm 0.0096$ & $0.8075 \pm 0.0075$ & $0.6319 \pm 0.0156$ \\
\midrule
\multirow{2}{*}{GAT} & Train Acc & $0.9886 \pm 0.0028$ & $0.9334 \pm 0.0062$ & $0.8487 \pm 0.0066$ & $0.9956 \pm 0.0019$ & $0.9535 \pm 0.0053$ & $0.8007 \pm 0.0091$ \\
 & Test Acc & $0.7715 \pm 0.0062$ & $0.8372 \pm 0.0123$ & $0.6938 \pm 0.0079$ & $0.7279 \pm 0.0124$ & $0.7879 \pm 0.0113$ & $0.6502 \pm 0.0179$ \\
\bottomrule
\end{tabular}
\end{table*}

%% file: tables/chameleon_all_accuracy_tables.tex
\begin{table*}[t]
\centering
\small
\caption{Performance on Chameleon with 10\% of nodes in the training set of the victim model under random and snowball sampling.}
\label{tab:performance_chameleon_train10}
\begin{tabular}{llcccccc}
\toprule
\multirow{2}{*}{Model} & \multirow{2}{*}{Metric} & \multicolumn{3}{c}{Random} & \multicolumn{3}{c}{Snowball} \\
\cmidrule(lr){3-5} \cmidrule(lr){6-8}
 &  & Original split & Full graph & No edges & Original split & Full graph & No edges \\
\midrule
\multirow{2}{*}{GCN} & Train Acc & $0.9454 \pm 0.0132$ & $0.5683 \pm 0.0251$ & $0.8626 \pm 0.0188$ & $0.9912 \pm 0.0074$ & $0.5366 \pm 0.0522$ & $0.8432 \pm 0.0389$ \\
 & Test Acc & $0.3706 \pm 0.0252$ & $0.4001 \pm 0.0236$ & $0.3421 \pm 0.0126$ & $0.3414 \pm 0.0205$ & $0.3952 \pm 0.0180$ & $0.3125 \pm 0.0155$ \\
\midrule
\multirow{2}{*}{GraphSAGE} & Train Acc & $0.9921 \pm 0.0051$ & $0.9330 \pm 0.0329$ & $0.8520 \pm 0.0132$ & $0.9956 \pm 0.0039$ & $0.9383 \pm 0.0169$ & $0.8634 \pm 0.0572$ \\
 & Test Acc & $0.3269 \pm 0.0103$ & $0.3468 \pm 0.0134$ & $0.3100 \pm 0.0205$ & $0.3216 \pm 0.0137$ & $0.3643 \pm 0.0243$ & $0.3230 \pm 0.0238$ \\
\midrule
\multirow{2}{*}{GAT} & Train Acc & $0.9445 \pm 0.0117$ & $0.5410 \pm 0.0660$ & $0.7463 \pm 0.0259$ & $0.9639 \pm 0.0085$ & $0.3894 \pm 0.0828$ & $0.6643 \pm 0.0525$ \\
 & Test Acc & $0.3329 \pm 0.0327$ & $0.3696 \pm 0.0429$ & $0.3296 \pm 0.0266$ & $0.3264 \pm 0.0392$ & $0.2799 \pm 0.0589$ & $0.2959 \pm 0.0244$ \\
\bottomrule
\end{tabular}
\end{table*}

\begin{table*}[t]
\centering
\small
\caption{Performance on Chameleon with 50\% of nodes in the training set of the victim model under random and snowball sampling.}
\label{tab:performance_chameleon_train50}
\begin{tabular}{llcccccc}
\toprule
\multirow{2}{*}{Model} & \multirow{2}{*}{Metric} & \multicolumn{3}{c}{Random} & \multicolumn{3}{c}{Snowball} \\
\cmidrule(lr){3-5} \cmidrule(lr){6-8}
 &  & Original split & Full graph & No edges & Original split & Full graph & No edges \\
\midrule
\multirow{2}{*}{GCN} & Train Acc & $0.8629 \pm 0.0106$ & $0.7409 \pm 0.0156$ & $0.7330 \pm 0.0131$ & $0.9661 \pm 0.0040$ & $0.6455 \pm 0.0081$ & $0.7418 \pm 0.0136$ \\
 & Test Acc & $0.4627 \pm 0.0209$ & $0.5905 \pm 0.0210$ & $0.3917 \pm 0.0136$ & $0.4114 \pm 0.0091$ & $0.5034 \pm 0.0229$ & $0.3874 \pm 0.0167$ \\
\midrule
\multirow{2}{*}{GraphSAGE} & Train Acc & $0.9960 \pm 0.0027$ & $0.9668 \pm 0.0092$ & $0.8062 \pm 0.0279$ & $0.9998 \pm 0.0004$ & $0.9649 \pm 0.0069$ & $0.8236 \pm 0.0267$ \\
 & Test Acc & $0.4369 \pm 0.0102$ & $0.5254 \pm 0.0085$ & $0.4093 \pm 0.0245$ & $0.3853 \pm 0.0377$ & $0.5157 \pm 0.0225$ & $0.3731 \pm 0.0416$ \\
\midrule
\multirow{2}{*}{GAT} & Train Acc & $0.9388 \pm 0.0084$ & $0.6480 \pm 0.0535$ & $0.5227 \pm 0.0140$ & $0.9757 \pm 0.0117$ & $0.4982 \pm 0.0299$ & $0.4996 \pm 0.0232$ \\
 & Test Acc & $0.2994 \pm 0.0252$ & $0.4973 \pm 0.0434$ & $0.3415 \pm 0.0158$ & $0.3579 \pm 0.0153$ & $0.4063 \pm 0.0340$ & $0.3426 \pm 0.0124$ \\
\bottomrule
\end{tabular}
\end{table*}

%% file: tables/pubmed_all_accuracy_tables.tex
\begin{table*}[t]
\centering
\small
\caption{Performance on PubMed with 10\% of nodes in the training set of the victim model under random and snowball sampling.}
\label{tab:performance_pubmed_train10}
\begin{tabular}{llcccccc}
\toprule
\multirow{2}{*}{Model} & \multirow{2}{*}{Metric} & \multicolumn{3}{c}{Random} & \multicolumn{3}{c}{Snowball} \\
\cmidrule(lr){3-5} \cmidrule(lr){6-8}
 &  & Original split & Full graph & No edges & Original split & Full graph & No edges \\
\midrule
\multirow{2}{*}{GCN} & Train Acc & $0.9559 \pm 0.0042$ & $0.8920 \pm 0.0054$ & $0.9526 \pm 0.0023$ & $0.9081 \pm 0.0037$ & $0.8988 \pm 0.0077$ & $0.8856 \pm 0.0069$ \\
 & Test Acc & $0.8590 \pm 0.0029$ & $0.8620 \pm 0.0030$ & $0.8285 \pm 0.0047$ & $0.8353 \pm 0.0027$ & $0.8450 \pm 0.0024$ & $0.8034 \pm 0.0049$ \\
\midrule
\multirow{2}{*}{GraphSAGE} & Train Acc & $0.9812 \pm 0.0019$ & $0.9371 \pm 0.0058$ & $0.9505 \pm 0.0036$ & $0.9943 \pm 0.0030$ & $0.9643 \pm 0.0072$ & $0.8202 \pm 0.0266$ \\
 & Test Acc & $0.8395 \pm 0.0091$ & $0.8405 \pm 0.0089$ & $0.8320 \pm 0.0034$ & $0.8036 \pm 0.0104$ & $0.8252 \pm 0.0111$ & $0.7172 \pm 0.0254$ \\
\midrule
\multirow{2}{*}{GAT} & Train Acc & $0.9651 \pm 0.0035$ & $0.8674 \pm 0.0084$ & $0.9562 \pm 0.0030$ & $0.9195 \pm 0.0081$ & $0.8918 \pm 0.0134$ & $0.8675 \pm 0.0080$ \\
 & Test Acc & $0.8343 \pm 0.0053$ & $0.8361 \pm 0.0051$ & $0.8211 \pm 0.0057$ & $0.8123 \pm 0.0086$ & $0.8224 \pm 0.0095$ & $0.7789 \pm 0.0077$ \\
\bottomrule
\end{tabular}
\end{table*}

\begin{table*}[t]
\centering
\small
\caption{Performance on PubMed with 50\% of nodes in the training set of the victim model under random and snowball sampling.}
\label{tab:performance_pubmed_train50}
\begin{tabular}{llcccccc}
\toprule
\multirow{2}{*}{Model} & \multirow{2}{*}{Metric} & \multicolumn{3}{c}{Random} & \multicolumn{3}{c}{Snowball} \\
\cmidrule(lr){3-5} \cmidrule(lr){6-8}
 &  & Original split & Full graph & No edges & Original split & Full graph & No edges \\
\midrule
\multirow{2}{*}{GCN} & Train Acc & $0.8920 \pm 0.0010$ & $0.8845 \pm 0.0023$ & $0.8875 \pm 0.0030$ & $0.8872 \pm 0.0024$ & $0.8885 \pm 0.0038$ & $0.8647 \pm 0.0020$ \\
 & Test Acc & $0.8691 \pm 0.0024$ & $0.8768 \pm 0.0023$ & $0.8613 \pm 0.0025$ & $0.8526 \pm 0.0027$ & $0.8712 \pm 0.0023$ & $0.8480 \pm 0.0011$ \\
\midrule
\multirow{2}{*}{GraphSAGE} & Train Acc & $0.9277 \pm 0.0023$ & $0.9200 \pm 0.0026$ & $0.8980 \pm 0.0017$ & $0.9337 \pm 0.0030$ & $0.9317 \pm 0.0025$ & $0.8555 \pm 0.0136$ \\
 & Test Acc & $0.8753 \pm 0.0027$ & $0.8854 \pm 0.0028$ & $0.8677 \pm 0.0015$ & $0.8361 \pm 0.0072$ & $0.8768 \pm 0.0022$ & $0.8239 \pm 0.0080$ \\
\midrule
\multirow{2}{*}{GAT} & Train Acc & $0.8897 \pm 0.0010$ & $0.8741 \pm 0.0029$ & $0.8873 \pm 0.0016$ & $0.8875 \pm 0.0033$ & $0.8869 \pm 0.0037$ & $0.8590 \pm 0.0034$ \\
 & Test Acc & $0.8623 \pm 0.0034$ & $0.8670 \pm 0.0027$ & $0.8598 \pm 0.0013$ & $0.8438 \pm 0.0021$ & $0.8589 \pm 0.0012$ & $0.8404 \pm 0.0028$ \\
\bottomrule
\end{tabular}
\end{table*}

%% file: tables/cora_all_gap_adv_tables.tex
\begin{table*}[htbp]
\centering
\caption{Performance gap (in \%) and membership advantage on \textsc{Cora} when $10\%$ of nodes are in the training set of the target model and the attacker train fraction is $r=0.1$.}
\label{tab:gap_adv_cora_train10_r10}
\setlength{\tabcolsep}{5pt}
\scriptsize
\resizebox{\textwidth}{!}{%
\begin{tabular}{@{}l l c c c c c c@{}}
\toprule
 &  & \multicolumn{3}{c}{\textbf{Random}} & \multicolumn{3}{c}{\textbf{Snowball}} \\
\cmidrule(lr){3-5} \cmidrule(lr){6-8}
\textbf{Model} & \textbf{Metric} & \makecell{\textsc{Original}\\\textsc{split}} & \makecell{\textsc{Full}\\\textsc{graph}} & \makecell{\textsc{No}\\\textsc{edges}} & \makecell{\textsc{Original}\\\textsc{split}} & \makecell{\textsc{Full}\\\textsc{graph}} & \makecell{\textsc{No}\\\textsc{edges}} \\
\midrule
\multirow{2}{*}{\textsc{GCN}} & \textsc{Perf. Gap} & $21.2257 \pm 1.0071$ & $16.6048 \pm 1.2985$ & $37.2089 \pm 1.3402$ & $29.1732 \pm 2.0505$ & $22.6257 \pm 2.3590$ & $43.9479 \pm 2.3569$ \\
 & \textsc{Adv.} & $0.7572 \pm 0.0521$ & $0.2448 \pm 0.0552$ & $0.6340 \pm 0.0598$ & $0.6564 \pm 0.0594$ & $0.4697 \pm 0.0515$ & $0.5159 \pm 0.0616$ \\
\midrule
\multirow{2}{*}{\textsc{GraphSage}} & \textsc{Perf. Gap} & $33.5275 \pm 2.6085$ & $32.5189 \pm 2.5481$ & $36.0882 \pm 0.9495$ & $29.6473 \pm 2.6414$ & $26.1806 \pm 2.1638$ & $42.0500 \pm 2.7268$ \\
 & \textsc{Adv.} & $0.6606 \pm 0.0511$ & $0.5880 \pm 0.0607$ & $0.5709 \pm 0.0625$ & $0.8430 \pm 0.0413$ & $0.6106 \pm 0.0234$ & $0.5340 \pm 0.0337$ \\
\midrule
\multirow{2}{*}{\textsc{GAT}} & \textsc{Perf. Gap} & $22.0094 \pm 1.6211$ & $15.6743 \pm 0.8820$ & $35.6151 \pm 1.3949$ & $27.8237 \pm 1.9008$ & $19.8598 \pm 1.4470$ & $35.4403 \pm 2.1669$ \\
 & \textsc{Adv.} & $0.7096 \pm 0.0682$ & $0.2167 \pm 0.0350$ & $0.6125 \pm 0.0637$ & $0.8268 \pm 0.0509$ & $0.4441 \pm 0.0628$ & $0.4228 \pm 0.0540$ \\
\bottomrule
\end{tabular}}
\end{table*}

\begin{table*}[htbp]
\centering
\caption{Performance gap (in \%) and membership advantage on \textsc{Cora} when $10\%$ of nodes are in the training set of the target model and the attacker train fraction is $r=0.2$.}
\label{tab:gap_adv_cora_train10_r20}
\setlength{\tabcolsep}{5pt}
\scriptsize
\resizebox{\textwidth}{!}{%
\begin{tabular}{@{}l l c c c c c c@{}}
\toprule
 &  & \multicolumn{3}{c}{\textbf{Random}} & \multicolumn{3}{c}{\textbf{Snowball}} \\
\cmidrule(lr){3-5} \cmidrule(lr){6-8}
\textbf{Model} & \textbf{Metric} & \makecell{\textsc{Original}\\\textsc{split}} & \makecell{\textsc{Full}\\\textsc{graph}} & \makecell{\textsc{No}\\\textsc{edges}} & \makecell{\textsc{Original}\\\textsc{split}} & \makecell{\textsc{Full}\\\textsc{graph}} & \makecell{\textsc{No}\\\textsc{edges}} \\
\midrule
\multirow{2}{*}{\textsc{GCN}} & \textsc{Perf. Gap} & $21.2257 \pm 1.0071$ & $16.6048 \pm 1.2985$ & $37.2089 \pm 1.3402$ & $29.1732 \pm 2.0505$ & $22.6257 \pm 2.3590$ & $43.9479 \pm 2.3569$ \\
 & \textsc{Adv.} & $0.8310 \pm 0.0217$ & $0.3059 \pm 0.0312$ & $0.6964 \pm 0.0320$ & $0.6801 \pm 0.0598$ & $0.5085 \pm 0.0379$ & $0.5470 \pm 0.0602$ \\
\midrule
\multirow{2}{*}{\textsc{GraphSage}} & \textsc{Perf. Gap} & $33.5275 \pm 2.6085$ & $32.5189 \pm 2.5481$ & $36.0882 \pm 0.9495$ & $29.6473 \pm 2.6414$ & $26.1806 \pm 2.1638$ & $42.0500 \pm 2.7268$ \\
 & \textsc{Adv.} & $0.7443 \pm 0.0363$ & $0.6555 \pm 0.0276$ & $0.6278 \pm 0.0385$ & $0.9122 \pm 0.0128$ & $0.6472 \pm 0.0175$ & $0.5426 \pm 0.0381$ \\
\midrule
\multirow{2}{*}{\textsc{GAT}} & \textsc{Perf. Gap} & $22.0094 \pm 1.6211$ & $15.6743 \pm 0.8820$ & $35.6151 \pm 1.3949$ & $27.8237 \pm 1.9008$ & $19.8598 \pm 1.4470$ & $35.4403 \pm 2.1669$ \\
 & \textsc{Adv.} & $0.8183 \pm 0.0225$ & $0.2648 \pm 0.0214$ & $0.6805 \pm 0.0365$ & $0.8918 \pm 0.0141$ & $0.5096 \pm 0.0446$ & $0.4576 \pm 0.0514$ \\
\bottomrule
\end{tabular}}
\end{table*}

\begin{table*}[htbp]
\centering
\caption{Performance gap (in \%) and membership advantage on \textsc{Cora} when $10\%$ of nodes are in the training set of the target model and the attacker train fraction is $r=0.5$.}
\label{tab:gap_adv_cora_train10_r50}
\setlength{\tabcolsep}{5pt}
\scriptsize
\resizebox{\textwidth}{!}{%
\begin{tabular}{@{}l l c c c c c c@{}}
\toprule
 &  & \multicolumn{3}{c}{\textbf{Random}} & \multicolumn{3}{c}{\textbf{Snowball}} \\
\cmidrule(lr){3-5} \cmidrule(lr){6-8}
\textbf{Model} & \textbf{Metric} & \makecell{\textsc{Original}\\\textsc{split}} & \makecell{\textsc{Full}\\\textsc{graph}} & \makecell{\textsc{No}\\\textsc{edges}} & \makecell{\textsc{Original}\\\textsc{split}} & \makecell{\textsc{Full}\\\textsc{graph}} & \makecell{\textsc{No}\\\textsc{edges}} \\
\midrule
\multirow{2}{*}{\textsc{GCN}} & \textsc{Perf. Gap} & $21.2257 \pm 1.0071$ & $16.6048 \pm 1.2985$ & $37.2089 \pm 1.3402$ & $29.1732 \pm 2.0505$ & $22.6257 \pm 2.3590$ & $43.9479 \pm 2.3569$ \\
 & \textsc{Adv.} & $0.8581 \pm 0.0222$ & $0.3347 \pm 0.0294$ & $0.7345 \pm 0.0430$ & $0.6951 \pm 0.0718$ & $0.5461 \pm 0.0565$ & $0.5792 \pm 0.0696$ \\
\midrule
\multirow{2}{*}{\textsc{GraphSage}} & \textsc{Perf. Gap} & $33.5275 \pm 2.6085$ & $32.5189 \pm 2.5481$ & $36.0882 \pm 0.9495$ & $29.6473 \pm 2.6414$ & $26.1806 \pm 2.1638$ & $42.0500 \pm 2.7268$ \\
 & \textsc{Adv.} & $0.7835 \pm 0.0292$ & $0.6876 \pm 0.0373$ & $0.6638 \pm 0.0454$ & $0.9411 \pm 0.0110$ & $0.6642 \pm 0.0278$ & $0.5700 \pm 0.0461$ \\
\midrule
\multirow{2}{*}{\textsc{GAT}} & \textsc{Perf. Gap} & $22.0094 \pm 1.6211$ & $15.6743 \pm 0.8820$ & $35.6151 \pm 1.3949$ & $27.8237 \pm 1.9008$ & $19.8598 \pm 1.4470$ & $35.4403 \pm 2.1669$ \\
 & \textsc{Adv.} & $0.8675 \pm 0.0180$ & $0.3088 \pm 0.0286$ & $0.7258 \pm 0.0413$ & $0.9154 \pm 0.0152$ & $0.5404 \pm 0.0466$ & $0.4954 \pm 0.0468$ \\
\bottomrule
\end{tabular}}
\end{table*}

\begin{table*}[htbp]
\centering
\caption{Performance gap (in \%) and membership advantage on \textsc{Cora} when $10\%$ of nodes are in the training set of the target model and the attacker train fraction is $r=0.8$.}
\label{tab:gap_adv_cora_train10_r80}
\setlength{\tabcolsep}{5pt}
\scriptsize
\resizebox{\textwidth}{!}{%
\begin{tabular}{@{}l l c c c c c c@{}}
\toprule
 &  & \multicolumn{3}{c}{\textbf{Random}} & \multicolumn{3}{c}{\textbf{Snowball}} \\
\cmidrule(lr){3-5} \cmidrule(lr){6-8}
\textbf{Model} & \textbf{Metric} & \makecell{\textsc{Original}\\\textsc{split}} & \makecell{\textsc{Full}\\\textsc{graph}} & \makecell{\textsc{No}\\\textsc{edges}} & \makecell{\textsc{Original}\\\textsc{split}} & \makecell{\textsc{Full}\\\textsc{graph}} & \makecell{\textsc{No}\\\textsc{edges}} \\
\midrule
\multirow{2}{*}{\textsc{GCN}} & \textsc{Perf. Gap} & $21.2257 \pm 1.0071$ & $16.6048 \pm 1.2985$ & $37.2089 \pm 1.3402$ & $29.1732 \pm 2.0505$ & $22.6257 \pm 2.3590$ & $43.9479 \pm 2.3569$ \\
 & \textsc{Adv.} & $0.8776 \pm 0.0312$ & $0.3498 \pm 0.0381$ & $0.7461 \pm 0.0458$ & $0.7309 \pm 0.0720$ & $0.6001 \pm 0.0695$ & $0.6136 \pm 0.0674$ \\
\midrule
\multirow{2}{*}{\textsc{GraphSage}} & \textsc{Perf. Gap} & $33.5275 \pm 2.6085$ & $32.5189 \pm 2.5481$ & $36.0882 \pm 0.9495$ & $29.6473 \pm 2.6414$ & $26.1806 \pm 2.1638$ & $42.0500 \pm 2.7268$ \\
 & \textsc{Adv.} & $0.7961 \pm 0.0437$ & $0.7052 \pm 0.0413$ & $0.6843 \pm 0.0545$ & $0.9462 \pm 0.0127$ & $0.6865 \pm 0.0415$ & $0.6006 \pm 0.0499$ \\
\midrule
\multirow{2}{*}{\textsc{GAT}} & \textsc{Perf. Gap} & $22.0094 \pm 1.6211$ & $15.6743 \pm 0.8820$ & $35.6151 \pm 1.3949$ & $27.8237 \pm 1.9008$ & $19.8598 \pm 1.4470$ & $35.4403 \pm 2.1669$ \\
 & \textsc{Adv.} & $0.8767 \pm 0.0205$ & $0.3468 \pm 0.0461$ & $0.7445 \pm 0.0449$ & $0.9270 \pm 0.0172$ & $0.5812 \pm 0.0959$ & $0.5302 \pm 0.0497$ \\
\bottomrule
\end{tabular}}
\end{table*}

\begin{table*}[htbp]
\centering
\caption{Performance gap (in \%) and membership advantage on \textsc{Cora} when $50\%$ of nodes are in the training set of the target model and the attacker train fraction is $r=0.1$.}
\label{tab:gap_adv_cora_train50_r10}
\setlength{\tabcolsep}{5pt}
\scriptsize
\resizebox{\textwidth}{!}{%
\begin{tabular}{@{}l l c c c c c c@{}}
\toprule
 &  & \multicolumn{3}{c}{\textbf{Random}} & \multicolumn{3}{c}{\textbf{Snowball}} \\
\cmidrule(lr){3-5} \cmidrule(lr){6-8}
\textbf{Model} & \textbf{Metric} & \makecell{\textsc{Original}\\\textsc{split}} & \makecell{\textsc{Full}\\\textsc{graph}} & \makecell{\textsc{No}\\\textsc{edges}} & \makecell{\textsc{Original}\\\textsc{split}} & \makecell{\textsc{Full}\\\textsc{graph}} & \makecell{\textsc{No}\\\textsc{edges}} \\
\midrule
\multirow{2}{*}{\textsc{GCN}} & \textsc{Perf. Gap} & $19.8832 \pm 0.8856$ & $8.2840 \pm 1.1413$ & $24.1190 \pm 1.5496$ & $27.3564 \pm 0.5142$ & $16.7243 \pm 1.3459$ & $27.6532 \pm 1.4966$ \\
 & \textsc{Adv.} & $0.2654 \pm 0.0217$ & $0.0822 \pm 0.0209$ & $0.2359 \pm 0.0260$ & $0.3814 \pm 0.0204$ & $0.2920 \pm 0.0271$ & $0.2726 \pm 0.0172$ \\
\midrule
\multirow{2}{*}{\textsc{GraphSage}} & \textsc{Perf. Gap} & $21.7282 \pm 0.7946$ & $15.0847 \pm 1.1030$ & $22.8597 \pm 1.8146$ & $25.4357 \pm 0.9648$ & $19.1270 \pm 0.7650$ & $28.0278 \pm 2.0823$ \\
 & \textsc{Adv.} & $0.4848 \pm 0.0243$ & $0.2651 \pm 0.0238$ & $0.2117 \pm 0.0218$ & $0.6307 \pm 0.0245$ & $0.4294 \pm 0.0283$ & $0.3091 \pm 0.0194$ \\
\midrule
\multirow{2}{*}{\textsc{GAT}} & \textsc{Perf. Gap} & $21.9638 \pm 0.5095$ & $10.3057 \pm 0.8758$ & $18.2564 \pm 0.6569$ & $26.8838 \pm 1.2589$ & $17.3588 \pm 1.5653$ & $18.7669 \pm 3.0288$ \\
 & \textsc{Adv.} & $0.4561 \pm 0.0312$ & $0.1064 \pm 0.0226$ & $0.1727 \pm 0.0357$ & $0.5680 \pm 0.0237$ & $0.3360 \pm 0.0244$ & $0.2119 \pm 0.0254$ \\
\bottomrule
\end{tabular}}
\end{table*}

\begin{table*}[htbp]
\centering
\caption{Performance gap (in \%) and membership advantage on \textsc{Cora} when $50\%$ of nodes are in the training set of the target model and the attacker train fraction is $r=0.2$.}
\label{tab:gap_adv_cora_train50_r20}
\setlength{\tabcolsep}{5pt}
\scriptsize
\resizebox{\textwidth}{!}{%
\begin{tabular}{@{}l l c c c c c c@{}}
\toprule
 &  & \multicolumn{3}{c}{\textbf{Random}} & \multicolumn{3}{c}{\textbf{Snowball}} \\
\cmidrule(lr){3-5} \cmidrule(lr){6-8}
\textbf{Model} & \textbf{Metric} & \makecell{\textsc{Original}\\\textsc{split}} & \makecell{\textsc{Full}\\\textsc{graph}} & \makecell{\textsc{No}\\\textsc{edges}} & \makecell{\textsc{Original}\\\textsc{split}} & \makecell{\textsc{Full}\\\textsc{graph}} & \makecell{\textsc{No}\\\textsc{edges}} \\
\midrule
\multirow{2}{*}{\textsc{GCN}} & \textsc{Perf. Gap} & $19.8832 \pm 0.8856$ & $8.2840 \pm 1.1413$ & $24.1190 \pm 1.5496$ & $27.3564 \pm 0.5142$ & $16.7243 \pm 1.3459$ & $27.6532 \pm 1.4966$ \\
 & \textsc{Adv.} & $0.2805 \pm 0.0194$ & $0.0862 \pm 0.0136$ & $0.2553 \pm 0.0171$ & $0.3879 \pm 0.0210$ & $0.3140 \pm 0.0183$ & $0.2877 \pm 0.0177$ \\
\midrule
\multirow{2}{*}{\textsc{GraphSage}} & \textsc{Perf. Gap} & $21.7282 \pm 0.7946$ & $15.0847 \pm 1.1030$ & $22.8597 \pm 1.8146$ & $25.4357 \pm 0.9648$ & $19.1270 \pm 0.7650$ & $28.0278 \pm 2.0823$ \\
 & \textsc{Adv.} & $0.5074 \pm 0.0195$ & $0.2880 \pm 0.0127$ & $0.2429 \pm 0.0174$ & $0.6426 \pm 0.0272$ & $0.4354 \pm 0.0315$ & $0.3135 \pm 0.0229$ \\
\midrule
\multirow{2}{*}{\textsc{GAT}} & \textsc{Perf. Gap} & $21.9638 \pm 0.5095$ & $10.3057 \pm 0.8758$ & $18.2564 \pm 0.6569$ & $26.8838 \pm 1.2589$ & $17.3588 \pm 1.5653$ & $18.7669 \pm 3.0288$ \\
 & \textsc{Adv.} & $0.4905 \pm 0.0129$ & $0.1148 \pm 0.0170$ & $0.2036 \pm 0.0209$ & $0.5806 \pm 0.0278$ & $0.3555 \pm 0.0244$ & $0.2336 \pm 0.0121$ \\
\bottomrule
\end{tabular}}
\end{table*}

\begin{table*}[htbp]
\centering
\caption{Performance gap (in \%) and membership advantage on \textsc{Cora} when $50\%$ of nodes are in the training set of the target model and the attacker train fraction is $r=0.5$.}
\label{tab:gap_adv_cora_train50_r50}
\setlength{\tabcolsep}{5pt}
\scriptsize
\resizebox{\textwidth}{!}{%
\begin{tabular}{@{}l l c c c c c c@{}}
\toprule
 &  & \multicolumn{3}{c}{\textbf{Random}} & \multicolumn{3}{c}{\textbf{Snowball}} \\
\cmidrule(lr){3-5} \cmidrule(lr){6-8}
\textbf{Model} & \textbf{Metric} & \makecell{\textsc{Original}\\\textsc{split}} & \makecell{\textsc{Full}\\\textsc{graph}} & \makecell{\textsc{No}\\\textsc{edges}} & \makecell{\textsc{Original}\\\textsc{split}} & \makecell{\textsc{Full}\\\textsc{graph}} & \makecell{\textsc{No}\\\textsc{edges}} \\
\midrule
\multirow{2}{*}{\textsc{GCN}} & \textsc{Perf. Gap} & $19.8832 \pm 0.8856$ & $8.2840 \pm 1.1413$ & $24.1190 \pm 1.5496$ & $27.3564 \pm 0.5142$ & $16.7243 \pm 1.3459$ & $27.6532 \pm 1.4966$ \\
 & \textsc{Adv.} & $0.2971 \pm 0.0210$ & $0.1004 \pm 0.0183$ & $0.2677 \pm 0.0231$ & $0.4056 \pm 0.0209$ & $0.3374 \pm 0.0170$ & $0.3045 \pm 0.0188$ \\
\midrule
\multirow{2}{*}{\textsc{GraphSage}} & \textsc{Perf. Gap} & $21.7282 \pm 0.7946$ & $15.0847 \pm 1.1030$ & $22.8597 \pm 1.8146$ & $25.4357 \pm 0.9648$ & $19.1270 \pm 0.7650$ & $28.0278 \pm 2.0823$ \\
 & \textsc{Adv.} & $0.5099 \pm 0.0212$ & $0.2987 \pm 0.0194$ & $0.2495 \pm 0.0156$ & $0.6530 \pm 0.0269$ & $0.4486 \pm 0.0300$ & $0.3285 \pm 0.0168$ \\
\midrule
\multirow{2}{*}{\textsc{GAT}} & \textsc{Perf. Gap} & $21.9638 \pm 0.5095$ & $10.3057 \pm 0.8758$ & $18.2564 \pm 0.6569$ & $26.8838 \pm 1.2589$ & $17.3588 \pm 1.5653$ & $18.7669 \pm 3.0288$ \\
 & \textsc{Adv.} & $0.5052 \pm 0.0129$ & $0.1383 \pm 0.0222$ & $0.2188 \pm 0.0188$ & $0.5993 \pm 0.0273$ & $0.3776 \pm 0.0310$ & $0.2529 \pm 0.0208$ \\
\bottomrule
\end{tabular}}
\end{table*}

\begin{table*}[htbp]
\centering
\caption{Performance gap (in \%) and membership advantage on \textsc{Cora} when $50\%$ of nodes are in the training set of the target model and the attacker train fraction is $r=0.8$.}
\label{tab:gap_adv_cora_train50_r80}
\setlength{\tabcolsep}{5pt}
\scriptsize
\resizebox{\textwidth}{!}{%
\begin{tabular}{@{}l l c c c c c c@{}}
\toprule
 &  & \multicolumn{3}{c}{\textbf{Random}} & \multicolumn{3}{c}{\textbf{Snowball}} \\
\cmidrule(lr){3-5} \cmidrule(lr){6-8}
\textbf{Model} & \textbf{Metric} & \makecell{\textsc{Original}\\\textsc{split}} & \makecell{\textsc{Full}\\\textsc{graph}} & \makecell{\textsc{No}\\\textsc{edges}} & \makecell{\textsc{Original}\\\textsc{split}} & \makecell{\textsc{Full}\\\textsc{graph}} & \makecell{\textsc{No}\\\textsc{edges}} \\
\midrule
\multirow{2}{*}{\textsc{GCN}} & \textsc{Perf. Gap} & $19.8832 \pm 0.8856$ & $8.2840 \pm 1.1413$ & $24.1190 \pm 1.5496$ & $27.3564 \pm 0.5142$ & $16.7243 \pm 1.3459$ & $27.6532 \pm 1.4966$ \\
 & \textsc{Adv.} & $0.3041 \pm 0.0360$ & $0.1200 \pm 0.0230$ & $0.2787 \pm 0.0285$ & $0.4325 \pm 0.0285$ & $0.3695 \pm 0.0317$ & $0.3205 \pm 0.0353$ \\
\midrule
\multirow{2}{*}{\textsc{GraphSage}} & \textsc{Perf. Gap} & $21.7282 \pm 0.7946$ & $15.0847 \pm 1.1030$ & $22.8597 \pm 1.8146$ & $25.4357 \pm 0.9648$ & $19.1270 \pm 0.7650$ & $28.0278 \pm 2.0823$ \\
 & \textsc{Adv.} & $0.5171 \pm 0.0297$ & $0.3092 \pm 0.0251$ & $0.2635 \pm 0.0278$ & $0.6549 \pm 0.0338$ & $0.4585 \pm 0.0411$ & $0.3535 \pm 0.0294$ \\
\midrule
\multirow{2}{*}{\textsc{GAT}} & \textsc{Perf. Gap} & $21.9638 \pm 0.5095$ & $10.3057 \pm 0.8758$ & $18.2564 \pm 0.6569$ & $26.8838 \pm 1.2589$ & $17.3588 \pm 1.5653$ & $18.7669 \pm 3.0288$ \\
 & \textsc{Adv.} & $0.5260 \pm 0.0247$ & $0.1601 \pm 0.0303$ & $0.2431 \pm 0.0296$ & $0.6089 \pm 0.0419$ & $0.3889 \pm 0.0344$ & $0.2797 \pm 0.0256$ \\
\bottomrule
\end{tabular}}
\end{table*}

%% file: tables/chameleon_all_gap_adv_tables.tex
\begin{table*}[htbp]
\centering
\caption{Performance gap (in \%) and membership advantage on \textsc{Chameleon} when $10\%$ of nodes are in the training set of the target model and the attacker train fraction is $r=0.1$.}
\label{tab:gap_adv_chameleon_train10_r10}
\setlength{\tabcolsep}{5pt}
\scriptsize
\resizebox{\textwidth}{!}{%
\begin{tabular}{@{}l l c c c c c c@{}}
\toprule
 &  & \multicolumn{3}{c}{\textbf{Random}} & \multicolumn{3}{c}{\textbf{Snowball}} \\
\cmidrule(lr){3-5} \cmidrule(lr){6-8}
\textbf{Model} & \textbf{Metric} & \makecell{\textsc{Original}\\\textsc{split}} & \makecell{\textsc{Full}\\\textsc{graph}} & \makecell{\textsc{No}\\\textsc{edges}} & \makecell{\textsc{Original}\\\textsc{split}} & \makecell{\textsc{Full}\\\textsc{graph}} & \makecell{\textsc{No}\\\textsc{edges}} \\
\midrule
\multirow{2}{*}{\textsc{GCN}} & \textsc{Perf. Gap} & $60.8093 \pm 2.3679$ & $29.5093 \pm 4.4010$ & $60.3211 \pm 1.5209$ & $65.5633 \pm 1.9914$ & $25.6865 \pm 7.3105$ & $62.9266 \pm 1.1975$ \\
 & \textsc{Adv.} & $0.7275 \pm 0.0445$ & $0.1287 \pm 0.0500$ & $0.5647 \pm 0.0640$ & $0.7986 \pm 0.0347$ & $0.1944 \pm 0.0636$ & $0.5467 \pm 0.0521$ \\
\midrule
\multirow{2}{*}{\textsc{GraphSage}} & \textsc{Perf. Gap} & $67.0425 \pm 1.1286$ & $62.7990 \pm 1.5861$ & $63.5992 \pm 2.4818$ & $67.7003 \pm 1.4069$ & $61.1871 \pm 2.3532$ & $62.5690 \pm 1.5817$ \\
 & \textsc{Adv.} & $0.9068 \pm 0.0479$ & $0.7020 \pm 0.0647$ & $0.6449 \pm 0.0494$ & $0.9483 \pm 0.0258$ & $0.6581 \pm 0.0414$ & $0.6280 \pm 0.0576$ \\
\midrule
\multirow{2}{*}{\textsc{GAT}} & \textsc{Perf. Gap} & $64.7719 \pm 3.2673$ & $31.5193 \pm 3.6788$ & $55.8904 \pm 2.3936$ & $66.1328 \pm 4.0300$ & $26.3182 \pm 13.9418$ & $55.3452 \pm 3.3219$ \\
 & \textsc{Adv.} & $0.8486 \pm 0.0418$ & $0.1398 \pm 0.0461$ & $0.4767 \pm 0.0621$ & $0.8209 \pm 0.0411$ & $0.1564 \pm 0.0508$ & $0.3897 \pm 0.0530$ \\
\bottomrule
\end{tabular}}
\end{table*}

\begin{table*}[htbp]
\centering
\caption{Performance gap (in \%) and membership advantage on \textsc{Chameleon} when $10\%$ of nodes are in the training set of the target model and the attacker train fraction is $r=0.2$.}
\label{tab:gap_adv_chameleon_train10_r20}
\setlength{\tabcolsep}{5pt}
\scriptsize
\resizebox{\textwidth}{!}{%
\begin{tabular}{@{}l l c c c c c c@{}}
\toprule
 &  & \multicolumn{3}{c}{\textbf{Random}} & \multicolumn{3}{c}{\textbf{Snowball}} \\
\cmidrule(lr){3-5} \cmidrule(lr){6-8}
\textbf{Model} & \textbf{Metric} & \makecell{\textsc{Original}\\\textsc{split}} & \makecell{\textsc{Full}\\\textsc{graph}} & \makecell{\textsc{No}\\\textsc{edges}} & \makecell{\textsc{Original}\\\textsc{split}} & \makecell{\textsc{Full}\\\textsc{graph}} & \makecell{\textsc{No}\\\textsc{edges}} \\
\midrule
\multirow{2}{*}{\textsc{GCN}} & \textsc{Perf. Gap} & $60.8093 \pm 2.3679$ & $29.5093 \pm 4.4010$ & $60.3211 \pm 1.5209$ & $65.5633 \pm 1.9914$ & $25.6865 \pm 7.3105$ & $62.9266 \pm 1.1975$ \\
 & \textsc{Adv.} & $0.7547 \pm 0.0475$ & $0.1537 \pm 0.0554$ & $0.6165 \pm 0.0323$ & $0.8078 \pm 0.0366$ & $0.2116 \pm 0.0666$ & $0.5845 \pm 0.0524$ \\
\midrule
\multirow{2}{*}{\textsc{GraphSage}} & \textsc{Perf. Gap} & $67.0425 \pm 1.1286$ & $62.7990 \pm 1.5861$ & $63.5992 \pm 2.4818$ & $67.7003 \pm 1.4069$ & $61.1871 \pm 2.3532$ & $62.5690 \pm 1.5817$ \\
 & \textsc{Adv.} & $0.9247 \pm 0.0248$ & $0.7156 \pm 0.0509$ & $0.6665 \pm 0.0254$ & $0.9561 \pm 0.0262$ & $0.6837 \pm 0.0408$ & $0.6580 \pm 0.0478$ \\
\midrule
\multirow{2}{*}{\textsc{GAT}} & \textsc{Perf. Gap} & $64.7719 \pm 3.2673$ & $31.5193 \pm 3.6788$ & $55.8904 \pm 2.3936$ & $66.1328 \pm 4.0300$ & $26.3182 \pm 13.9418$ & $55.3452 \pm 3.3219$ \\
 & \textsc{Adv.} & $0.8661 \pm 0.0342$ & $0.1582 \pm 0.0467$ & $0.5308 \pm 0.0347$ & $0.8371 \pm 0.0432$ & $0.2064 \pm 0.0491$ & $0.4321 \pm 0.0426$ \\
\bottomrule
\end{tabular}}
\end{table*}

\begin{table*}[htbp]
\centering
\caption{Performance gap (in \%) and membership advantage on \textsc{Chameleon} when $10\%$ of nodes are in the training set of the target model and the attacker train fraction is $r=0.5$.}
\label{tab:gap_adv_chameleon_train10_r50}
\setlength{\tabcolsep}{5pt}
\scriptsize
\resizebox{\textwidth}{!}{%
\begin{tabular}{@{}l l c c c c c c@{}}
\toprule
 &  & \multicolumn{3}{c}{\textbf{Random}} & \multicolumn{3}{c}{\textbf{Snowball}} \\
\cmidrule(lr){3-5} \cmidrule(lr){6-8}
\textbf{Model} & \textbf{Metric} & \makecell{\textsc{Original}\\\textsc{split}} & \makecell{\textsc{Full}\\\textsc{graph}} & \makecell{\textsc{No}\\\textsc{edges}} & \makecell{\textsc{Original}\\\textsc{split}} & \makecell{\textsc{Full}\\\textsc{graph}} & \makecell{\textsc{No}\\\textsc{edges}} \\
\midrule
\multirow{2}{*}{\textsc{GCN}} & \textsc{Perf. Gap} & $60.8093 \pm 2.3679$ & $29.5093 \pm 4.4010$ & $60.3211 \pm 1.5209$ & $65.5633 \pm 1.9914$ & $25.6865 \pm 7.3105$ & $62.9266 \pm 1.1975$ \\
 & \textsc{Adv.} & $0.8101 \pm 0.0516$ & $0.1986 \pm 0.0276$ & $0.6247 \pm 0.0223$ & $0.8162 \pm 0.0341$ & $0.2827 \pm 0.0693$ & $0.6073 \pm 0.0481$ \\
\midrule
\multirow{2}{*}{\textsc{GraphSage}} & \textsc{Perf. Gap} & $67.0425 \pm 1.1286$ & $62.7990 \pm 1.5861$ & $63.5992 \pm 2.4818$ & $67.7003 \pm 1.4069$ & $61.1871 \pm 2.3532$ & $62.5690 \pm 1.5817$ \\
 & \textsc{Adv.} & $0.9358 \pm 0.0221$ & $0.7320 \pm 0.0550$ & $0.6915 \pm 0.0513$ & $0.9611 \pm 0.0226$ & $0.7051 \pm 0.0414$ & $0.6774 \pm 0.0607$ \\
\midrule
\multirow{2}{*}{\textsc{GAT}} & \textsc{Perf. Gap} & $64.7719 \pm 3.2673$ & $31.5193 \pm 3.6788$ & $55.8904 \pm 2.3936$ & $66.1328 \pm 4.0300$ & $26.3182 \pm 13.9418$ & $55.3452 \pm 3.3219$ \\
 & \textsc{Adv.} & $0.9163 \pm 0.0374$ & $0.1866 \pm 0.0308$ & $0.5385 \pm 0.0445$ & $0.8592 \pm 0.0513$ & $0.2425 \pm 0.0592$ & $0.4490 \pm 0.0568$ \\
\bottomrule
\end{tabular}}
\end{table*}

\begin{table*}[htbp]
\centering
\caption{Performance gap (in \%) and membership advantage on \textsc{Chameleon} when $10\%$ of nodes are in the training set of the target model and the attacker train fraction is $r=0.8$.}
\label{tab:gap_adv_chameleon_train10_r80}
\setlength{\tabcolsep}{5pt}
\scriptsize
\resizebox{\textwidth}{!}{%
\begin{tabular}{@{}l l c c c c c c@{}}
\toprule
 &  & \multicolumn{3}{c}{\textbf{Random}} & \multicolumn{3}{c}{\textbf{Snowball}} \\
\cmidrule(lr){3-5} \cmidrule(lr){6-8}
\textbf{Model} & \textbf{Metric} & \makecell{\textsc{Original}\\\textsc{split}} & \makecell{\textsc{Full}\\\textsc{graph}} & \makecell{\textsc{No}\\\textsc{edges}} & \makecell{\textsc{Original}\\\textsc{split}} & \makecell{\textsc{Full}\\\textsc{graph}} & \makecell{\textsc{No}\\\textsc{edges}} \\
\midrule
\multirow{2}{*}{\textsc{GCN}} & \textsc{Perf. Gap} & $60.8093 \pm 2.3679$ & $29.5093 \pm 4.4010$ & $60.3211 \pm 1.5209$ & $65.5633 \pm 1.9914$ & $25.6865 \pm 7.3105$ & $62.9266 \pm 1.1975$ \\
 & \textsc{Adv.} & $0.8500 \pm 0.0403$ & $0.2053 \pm 0.0391$ & $0.6315 \pm 0.0483$ & $0.8360 \pm 0.0472$ & $0.3281 \pm 0.0901$ & $0.6202 \pm 0.0616$ \\
\midrule
\multirow{2}{*}{\textsc{GraphSage}} & \textsc{Perf. Gap} & $67.0425 \pm 1.1286$ & $62.7990 \pm 1.5861$ & $63.5992 \pm 2.4818$ & $67.7003 \pm 1.4069$ & $61.1871 \pm 2.3532$ & $62.5690 \pm 1.5817$ \\
 & \textsc{Adv.} & $0.9468 \pm 0.0212$ & $0.7308 \pm 0.0553$ & $0.6947 \pm 0.0692$ & $0.9753 \pm 0.0182$ & $0.7201 \pm 0.0538$ & $0.7145 \pm 0.0525$ \\
\midrule
\multirow{2}{*}{\textsc{GAT}} & \textsc{Perf. Gap} & $64.7719 \pm 3.2673$ & $31.5193 \pm 3.6788$ & $55.8904 \pm 2.3936$ & $66.1328 \pm 4.0300$ & $26.3182 \pm 13.9418$ & $55.3452 \pm 3.3219$ \\
 & \textsc{Adv.} & $0.9377 \pm 0.0400$ & $0.2199 \pm 0.0473$ & $0.5648 \pm 0.0704$ & $0.8752 \pm 0.0362$ & $0.2527 \pm 0.1014$ & $0.4767 \pm 0.0635$ \\
\bottomrule
\end{tabular}}
\end{table*}

\begin{table*}[htbp]
\centering
\caption{Performance gap (in \%) and membership advantage on \textsc{Chameleon} when $50\%$ of nodes are in the training set of the target model and the attacker train fraction is $r=0.1$.}
\label{tab:gap_adv_chameleon_train50_r10}
\setlength{\tabcolsep}{5pt}
\scriptsize
\resizebox{\textwidth}{!}{%
\begin{tabular}{@{}l l c c c c c c@{}}
\toprule
 &  & \multicolumn{3}{c}{\textbf{Random}} & \multicolumn{3}{c}{\textbf{Snowball}} \\
\cmidrule(lr){3-5} \cmidrule(lr){6-8}
\textbf{Model} & \textbf{Metric} & \makecell{\textsc{Original}\\\textsc{split}} & \makecell{\textsc{Full}\\\textsc{graph}} & \makecell{\textsc{No}\\\textsc{edges}} & \makecell{\textsc{Original}\\\textsc{split}} & \makecell{\textsc{Full}\\\textsc{graph}} & \makecell{\textsc{No}\\\textsc{edges}} \\
\midrule
\multirow{2}{*}{\textsc{GCN}} & \textsc{Perf. Gap} & $46.3581 \pm 2.8051$ & $20.2978 \pm 2.4804$ & $46.5397 \pm 2.1297$ & $57.4159 \pm 0.8470$ & $22.0002 \pm 3.6979$ & $47.7655 \pm 2.4608$ \\
 & \textsc{Adv.} & $0.4723 \pm 0.0697$ & $0.1303 \pm 0.0169$ & $0.3501 \pm 0.0268$ & $0.6154 \pm 0.0202$ & $0.2699 \pm 0.0399$ & $0.3672 \pm 0.0378$ \\
\midrule
\multirow{2}{*}{\textsc{GraphSage}} & \textsc{Perf. Gap} & $56.1335 \pm 1.0929$ & $45.6566 \pm 0.7884$ & $49.2196 \pm 2.5772$ & $61.4682 \pm 3.7748$ & $46.5315 \pm 2.6960$ & $54.7559 \pm 4.0392$ \\
 & \textsc{Adv.} & $0.7947 \pm 0.0175$ & $0.5735 \pm 0.0262$ & $0.4645 \pm 0.0194$ & $0.8529 \pm 0.0222$ & $0.5577 \pm 0.0351$ & $0.4853 \pm 0.0166$ \\
\midrule
\multirow{2}{*}{\textsc{GAT}} & \textsc{Perf. Gap} & $68.1249 \pm 2.5124$ & $23.2490 \pm 2.2998$ & $34.6459 \pm 2.7079$ & $63.3213 \pm 1.5763$ & $18.4210 \pm 5.1792$ & $31.3578 \pm 2.5447$ \\
 & \textsc{Adv.} & $0.7513 \pm 0.0309$ & $0.1217 \pm 0.0289$ & $0.1783 \pm 0.0319$ & $0.7021 \pm 0.0395$ & $0.1562 \pm 0.0450$ & $0.1724 \pm 0.0285$ \\
\bottomrule
\end{tabular}}
\end{table*}

\begin{table*}[htbp]
\centering
\caption{Performance gap (in \%) and membership advantage on \textsc{Chameleon} when $50\%$ of nodes are in the training set of the target model and the attacker train fraction is $r=0.2$.}
\label{tab:gap_adv_chameleon_train50_r20}
\setlength{\tabcolsep}{5pt}
\scriptsize
\resizebox{\textwidth}{!}{%
\begin{tabular}{@{}l l c c c c c c@{}}
\toprule
 &  & \multicolumn{3}{c}{\textbf{Random}} & \multicolumn{3}{c}{\textbf{Snowball}} \\
\cmidrule(lr){3-5} \cmidrule(lr){6-8}
\textbf{Model} & \textbf{Metric} & \makecell{\textsc{Original}\\\textsc{split}} & \makecell{\textsc{Full}\\\textsc{graph}} & \makecell{\textsc{No}\\\textsc{edges}} & \makecell{\textsc{Original}\\\textsc{split}} & \makecell{\textsc{Full}\\\textsc{graph}} & \makecell{\textsc{No}\\\textsc{edges}} \\
\midrule
\multirow{2}{*}{\textsc{GCN}} & \textsc{Perf. Gap} & $46.3581 \pm 2.8051$ & $20.2978 \pm 2.4804$ & $46.5397 \pm 2.1297$ & $57.4159 \pm 0.8470$ & $22.0002 \pm 3.6979$ & $47.7655 \pm 2.4608$ \\
 & \textsc{Adv.} & $0.5142 \pm 0.0665$ & $0.1330 \pm 0.0221$ & $0.3589 \pm 0.0162$ & $0.6264 \pm 0.0296$ & $0.3019 \pm 0.0434$ & $0.3920 \pm 0.0271$ \\
\midrule
\multirow{2}{*}{\textsc{GraphSage}} & \textsc{Perf. Gap} & $56.1335 \pm 1.0929$ & $45.6566 \pm 0.7884$ & $49.2196 \pm 2.5772$ & $61.4682 \pm 3.7748$ & $46.5315 \pm 2.6960$ & $54.7559 \pm 4.0392$ \\
 & \textsc{Adv.} & $0.8094 \pm 0.0142$ & $0.5839 \pm 0.0185$ & $0.4902 \pm 0.0162$ & $0.8630 \pm 0.0219$ & $0.5688 \pm 0.0293$ & $0.5047 \pm 0.0196$ \\
\midrule
\multirow{2}{*}{\textsc{GAT}} & \textsc{Perf. Gap} & $68.1249 \pm 2.5124$ & $23.2490 \pm 2.2998$ & $34.6459 \pm 2.7079$ & $63.3213 \pm 1.5763$ & $18.4210 \pm 5.1792$ & $31.3578 \pm 2.5447$ \\
 & \textsc{Adv.} & $0.7641 \pm 0.0314$ & $0.1273 \pm 0.0316$ & $0.1877 \pm 0.0204$ & $0.7109 \pm 0.0405$ & $0.1862 \pm 0.0547$ & $0.1988 \pm 0.0302$ \\
\bottomrule
\end{tabular}}
\end{table*}

\begin{table*}[htbp]
\centering
\caption{Performance gap (in \%) and membership advantage on \textsc{Chameleon} when $50\%$ of nodes are in the training set of the target model and the attacker train fraction is $r=0.5$.}
\label{tab:gap_adv_chameleon_train50_r50}
\setlength{\tabcolsep}{5pt}
\scriptsize
\resizebox{\textwidth}{!}{%
\begin{tabular}{@{}l l c c c c c c@{}}
\toprule
 &  & \multicolumn{3}{c}{\textbf{Random}} & \multicolumn{3}{c}{\textbf{Snowball}} \\
\cmidrule(lr){3-5} \cmidrule(lr){6-8}
\textbf{Model} & \textbf{Metric} & \makecell{\textsc{Original}\\\textsc{split}} & \makecell{\textsc{Full}\\\textsc{graph}} & \makecell{\textsc{No}\\\textsc{edges}} & \makecell{\textsc{Original}\\\textsc{split}} & \makecell{\textsc{Full}\\\textsc{graph}} & \makecell{\textsc{No}\\\textsc{edges}} \\
\midrule
\multirow{2}{*}{\textsc{GCN}} & \textsc{Perf. Gap} & $46.3581 \pm 2.8051$ & $20.2978 \pm 2.4804$ & $46.5397 \pm 2.1297$ & $57.4159 \pm 0.8470$ & $22.0002 \pm 3.6979$ & $47.7655 \pm 2.4608$ \\
 & \textsc{Adv.} & $0.5361 \pm 0.0684$ & $0.1519 \pm 0.0237$ & $0.3597 \pm 0.0282$ & $0.6548 \pm 0.0197$ & $0.3586 \pm 0.0238$ & $0.4161 \pm 0.0308$ \\
\midrule
\multirow{2}{*}{\textsc{GraphSage}} & \textsc{Perf. Gap} & $56.1335 \pm 1.0929$ & $45.6566 \pm 0.7884$ & $49.2196 \pm 2.5772$ & $61.4682 \pm 3.7748$ & $46.5315 \pm 2.6960$ & $54.7559 \pm 4.0392$ \\
 & \textsc{Adv.} & $0.8206 \pm 0.0127$ & $0.5916 \pm 0.0201$ & $0.5038 \pm 0.0222$ & $0.8713 \pm 0.0214$ & $0.5803 \pm 0.0335$ & $0.5263 \pm 0.0212$ \\
\midrule
\multirow{2}{*}{\textsc{GAT}} & \textsc{Perf. Gap} & $68.1249 \pm 2.5124$ & $23.2490 \pm 2.2998$ & $34.6459 \pm 2.7079$ & $63.3213 \pm 1.5763$ & $18.4210 \pm 5.1792$ & $31.3578 \pm 2.5447$ \\
 & \textsc{Adv.} & $0.7873 \pm 0.0352$ & $0.1495 \pm 0.0301$ & $0.2046 \pm 0.0301$ & $0.7247 \pm 0.0382$ & $0.2203 \pm 0.0417$ & $0.2174 \pm 0.0200$ \\
\bottomrule
\end{tabular}}
\end{table*}

\begin{table*}[htbp]
\centering
\caption{Performance gap (in \%) and membership advantage on \textsc{Chameleon} when $50\%$ of nodes are in the training set of the target model and the attacker train fraction is $r=0.8$.}
\label{tab:gap_adv_chameleon_train50_r80}
\setlength{\tabcolsep}{5pt}
\scriptsize
\resizebox{\textwidth}{!}{%
\begin{tabular}{@{}l l c c c c c c@{}}
\toprule
 &  & \multicolumn{3}{c}{\textbf{Random}} & \multicolumn{3}{c}{\textbf{Snowball}} \\
\cmidrule(lr){3-5} \cmidrule(lr){6-8}
\textbf{Model} & \textbf{Metric} & \makecell{\textsc{Original}\\\textsc{split}} & \makecell{\textsc{Full}\\\textsc{graph}} & \makecell{\textsc{No}\\\textsc{edges}} & \makecell{\textsc{Original}\\\textsc{split}} & \makecell{\textsc{Full}\\\textsc{graph}} & \makecell{\textsc{No}\\\textsc{edges}} \\
\midrule
\multirow{2}{*}{\textsc{GCN}} & \textsc{Perf. Gap} & $46.3581 \pm 2.8051$ & $20.2978 \pm 2.4804$ & $46.5397 \pm 2.1297$ & $57.4159 \pm 0.8470$ & $22.0002 \pm 3.6979$ & $47.7655 \pm 2.4608$ \\
 & \textsc{Adv.} & $0.5488 \pm 0.0702$ & $0.1711 \pm 0.0378$ & $0.3789 \pm 0.0346$ & $0.6623 \pm 0.0199$ & $0.3781 \pm 0.0405$ & $0.4099 \pm 0.0341$ \\
\midrule
\multirow{2}{*}{\textsc{GraphSage}} & \textsc{Perf. Gap} & $56.1335 \pm 1.0929$ & $45.6566 \pm 0.7884$ & $49.2196 \pm 2.5772$ & $61.4682 \pm 3.7748$ & $46.5315 \pm 2.6960$ & $54.7559 \pm 4.0392$ \\
 & \textsc{Adv.} & $0.8243 \pm 0.0225$ & $0.5971 \pm 0.0226$ & $0.5111 \pm 0.0338$ & $0.8801 \pm 0.0260$ & $0.5915 \pm 0.0251$ & $0.5377 \pm 0.0217$ \\
\midrule
\multirow{2}{*}{\textsc{GAT}} & \textsc{Perf. Gap} & $68.1249 \pm 2.5124$ & $23.2490 \pm 2.2998$ & $34.6459 \pm 2.7079$ & $63.3213 \pm 1.5763$ & $18.4210 \pm 5.1792$ & $31.3578 \pm 2.5447$ \\
 & \textsc{Adv.} & $0.8044 \pm 0.0388$ & $0.1637 \pm 0.0402$ & $0.2164 \pm 0.0341$ & $0.7310 \pm 0.0300$ & $0.2515 \pm 0.0526$ & $0.2132 \pm 0.0267$ \\
\bottomrule
\end{tabular}}
\end{table*}

%% file: tables/pubmed_all_gap_adv_tables.tex
\begin{table*}[htbp]
\centering
\caption{Performance gap (in \%) and membership advantage on \textsc{Pubmed} when $10\%$ of nodes are in the training set of the target model and the attacker train fraction is $r=0.1$.}
\label{tab:gap_adv_pubmed_train10_r10}
\setlength{\tabcolsep}{5pt}
\scriptsize
\resizebox{\textwidth}{!}{%
\begin{tabular}{@{}l l c c c c c c@{}}
\toprule
 &  & \multicolumn{3}{c}{\textbf{Random}} & \multicolumn{3}{c}{\textbf{Snowball}} \\
\cmidrule(lr){3-5} \cmidrule(lr){6-8}
\textbf{Model} & \textbf{Metric} & \makecell{\textsc{Original}\\\textsc{split}} & \makecell{\textsc{Full}\\\textsc{graph}} & \makecell{\textsc{No}\\\textsc{edges}} & \makecell{\textsc{Original}\\\textsc{split}} & \makecell{\textsc{Full}\\\textsc{graph}} & \makecell{\textsc{No}\\\textsc{edges}} \\
\midrule
\multirow{2}{*}{\textsc{GCN}} & \textsc{Perf. Gap} & $10.1342 \pm 0.4471$ & $3.3656 \pm 0.5857$ & $13.0300 \pm 0.3600$ & $8.0153 \pm 0.3137$ & $5.9776 \pm 0.7850$ & $9.2859 \pm 0.8411$ \\
 & \textsc{Adv.} & $0.2686 \pm 0.0207$ & $0.0333 \pm 0.0083$ & $0.1344 \pm 0.0081$ & $0.2659 \pm 0.0762$ & $0.3243 \pm 0.0643$ & $0.1627 \pm 0.0491$ \\
\midrule
\multirow{2}{*}{\textsc{GraphSage}} & \textsc{Perf. Gap} & $14.4430 \pm 0.8746$ & $10.3028 \pm 0.7984$ & $12.4606 \pm 0.3359$ & $19.1853 \pm 0.8606$ & $14.4240 \pm 0.6453$ & $12.5284 \pm 2.5280$ \\
 & \textsc{Adv.} & $0.1996 \pm 0.0150$ & $0.1030 \pm 0.0139$ & $0.1310 \pm 0.0135$ & $0.3206 \pm 0.0473$ & $0.2259 \pm 0.0548$ & $0.1512 \pm 0.0412$ \\
\midrule
\multirow{2}{*}{\textsc{GAT}} & \textsc{Perf. Gap} & $13.5540 \pm 0.4784$ & $3.6067 \pm 0.6083$ & $14.1310 \pm 0.4605$ & $11.6681 \pm 0.2751$ & $7.7704 \pm 1.2851$ & $10.2044 \pm 1.0122$ \\
 & \textsc{Adv.} & $0.1615 \pm 0.0181$ & $0.0301 \pm 0.0098$ & $0.1502 \pm 0.0073$ & $0.2267 \pm 0.0658$ & $0.1912 \pm 0.0656$ & $0.1527 \pm 0.0529$ \\
\bottomrule
\end{tabular}}
\end{table*}

\begin{table*}[htbp]
\centering
\caption{Performance gap (in \%) and membership advantage on \textsc{Pubmed} when $10\%$ of nodes are in the training set of the target model and the attacker train fraction is $r=0.2$.}
\label{tab:gap_adv_pubmed_train10_r20}
\setlength{\tabcolsep}{5pt}
\scriptsize
\resizebox{\textwidth}{!}{%
\begin{tabular}{@{}l l c c c c c c@{}}
\toprule
 &  & \multicolumn{3}{c}{\textbf{Random}} & \multicolumn{3}{c}{\textbf{Snowball}} \\
\cmidrule(lr){3-5} \cmidrule(lr){6-8}
\textbf{Model} & \textbf{Metric} & \makecell{\textsc{Original}\\\textsc{split}} & \makecell{\textsc{Full}\\\textsc{graph}} & \makecell{\textsc{No}\\\textsc{edges}} & \makecell{\textsc{Original}\\\textsc{split}} & \makecell{\textsc{Full}\\\textsc{graph}} & \makecell{\textsc{No}\\\textsc{edges}} \\
\midrule
\multirow{2}{*}{\textsc{GCN}} & \textsc{Perf. Gap} & $10.1342 \pm 0.4471$ & $3.3656 \pm 0.5857$ & $13.0300 \pm 0.3600$ & $8.0153 \pm 0.3137$ & $5.9776 \pm 0.7850$ & $9.2859 \pm 0.8411$ \\
 & \textsc{Adv.} & $0.3215 \pm 0.0230$ & $0.0311 \pm 0.0072$ & $0.1394 \pm 0.0090$ & $0.2774 \pm 0.0727$ & $0.3384 \pm 0.0612$ & $0.1732 \pm 0.0515$ \\
\midrule
\multirow{2}{*}{\textsc{GraphSage}} & \textsc{Perf. Gap} & $14.4430 \pm 0.8746$ & $10.3028 \pm 0.7984$ & $12.4606 \pm 0.3359$ & $19.1853 \pm 0.8606$ & $14.4240 \pm 0.6453$ & $12.5284 \pm 2.5280$ \\
 & \textsc{Adv.} & $0.2099 \pm 0.0177$ & $0.1004 \pm 0.0112$ & $0.1368 \pm 0.0122$ & $0.3255 \pm 0.0488$ & $0.2348 \pm 0.0524$ & $0.1623 \pm 0.0333$ \\
\midrule
\multirow{2}{*}{\textsc{GAT}} & \textsc{Perf. Gap} & $13.5540 \pm 0.4784$ & $3.6067 \pm 0.6083$ & $14.1310 \pm 0.4605$ & $11.6681 \pm 0.2751$ & $7.7704 \pm 1.2851$ & $10.2044 \pm 1.0122$ \\
 & \textsc{Adv.} & $0.1825 \pm 0.0159$ & $0.0313 \pm 0.0076$ & $0.1523 \pm 0.0077$ & $0.2377 \pm 0.0621$ & $0.2107 \pm 0.0551$ & $0.1606 \pm 0.0540$ \\
\bottomrule
\end{tabular}}
\end{table*}

\begin{table*}[htbp]
\centering
\caption{Performance gap (in \%) and membership advantage on \textsc{Pubmed} when $10\%$ of nodes are in the training set of the target model and the attacker train fraction is $r=0.5$.}
\label{tab:gap_adv_pubmed_train10_r50}
\setlength{\tabcolsep}{5pt}
\scriptsize
\resizebox{\textwidth}{!}{%
\begin{tabular}{@{}l l c c c c c c@{}}
\toprule
 &  & \multicolumn{3}{c}{\textbf{Random}} & \multicolumn{3}{c}{\textbf{Snowball}} \\
\cmidrule(lr){3-5} \cmidrule(lr){6-8}
\textbf{Model} & \textbf{Metric} & \makecell{\textsc{Original}\\\textsc{split}} & \makecell{\textsc{Full}\\\textsc{graph}} & \makecell{\textsc{No}\\\textsc{edges}} & \makecell{\textsc{Original}\\\textsc{split}} & \makecell{\textsc{Full}\\\textsc{graph}} & \makecell{\textsc{No}\\\textsc{edges}} \\
\midrule
\multirow{2}{*}{\textsc{GCN}} & \textsc{Perf. Gap} & $10.1342 \pm 0.4471$ & $3.3656 \pm 0.5857$ & $13.0300 \pm 0.3600$ & $8.0153 \pm 0.3137$ & $5.9776 \pm 0.7850$ & $9.2859 \pm 0.8411$ \\
 & \textsc{Adv.} & $0.3752 \pm 0.0181$ & $0.0457 \pm 0.0094$ & $0.1444 \pm 0.0093$ & $0.2864 \pm 0.0685$ & $0.3503 \pm 0.0622$ & $0.1819 \pm 0.0530$ \\
\midrule
\multirow{2}{*}{\textsc{GraphSage}} & \textsc{Perf. Gap} & $14.4430 \pm 0.8746$ & $10.3028 \pm 0.7984$ & $12.4606 \pm 0.3359$ & $19.1853 \pm 0.8606$ & $14.4240 \pm 0.6453$ & $12.5284 \pm 2.5280$ \\
 & \textsc{Adv.} & $0.2326 \pm 0.0250$ & $0.1141 \pm 0.0100$ & $0.1433 \pm 0.0099$ & $0.3252 \pm 0.0494$ & $0.2370 \pm 0.0537$ & $0.1731 \pm 0.0319$ \\
\midrule
\multirow{2}{*}{\textsc{GAT}} & \textsc{Perf. Gap} & $13.5540 \pm 0.4784$ & $3.6067 \pm 0.6083$ & $14.1310 \pm 0.4605$ & $11.6681 \pm 0.2751$ & $7.7704 \pm 1.2851$ & $10.2044 \pm 1.0122$ \\
 & \textsc{Adv.} & $0.2107 \pm 0.0193$ & $0.0381 \pm 0.0083$ & $0.1620 \pm 0.0082$ & $0.2487 \pm 0.0634$ & $0.2237 \pm 0.0548$ & $0.1672 \pm 0.0528$ \\
\bottomrule
\end{tabular}}
\end{table*}

\begin{table*}[htbp]
\centering
\caption{Performance gap (in \%) and membership advantage on \textsc{Pubmed} when $10\%$ of nodes are in the training set of the target model and the attacker train fraction is $r=0.8$.}
\label{tab:gap_adv_pubmed_train10_r80}
\setlength{\tabcolsep}{5pt}
\scriptsize
\resizebox{\textwidth}{!}{%
\begin{tabular}{@{}l l c c c c c c@{}}
\toprule
 &  & \multicolumn{3}{c}{\textbf{Random}} & \multicolumn{3}{c}{\textbf{Snowball}} \\
\cmidrule(lr){3-5} \cmidrule(lr){6-8}
\textbf{Model} & \textbf{Metric} & \makecell{\textsc{Original}\\\textsc{split}} & \makecell{\textsc{Full}\\\textsc{graph}} & \makecell{\textsc{No}\\\textsc{edges}} & \makecell{\textsc{Original}\\\textsc{split}} & \makecell{\textsc{Full}\\\textsc{graph}} & \makecell{\textsc{No}\\\textsc{edges}} \\
\midrule
\multirow{2}{*}{\textsc{GCN}} & \textsc{Perf. Gap} & $10.1342 \pm 0.4471$ & $3.3656 \pm 0.5857$ & $13.0300 \pm 0.3600$ & $8.0153 \pm 0.3137$ & $5.9776 \pm 0.7850$ & $9.2859 \pm 0.8411$ \\
 & \textsc{Adv.} & $0.3930 \pm 0.0303$ & $0.0581 \pm 0.0146$ & $0.1618 \pm 0.0156$ & $0.3051 \pm 0.0716$ & $0.3659 \pm 0.0633$ & $0.1841 \pm 0.0631$ \\
\midrule
\multirow{2}{*}{\textsc{GraphSage}} & \textsc{Perf. Gap} & $14.4430 \pm 0.8746$ & $10.3028 \pm 0.7984$ & $12.4606 \pm 0.3359$ & $19.1853 \pm 0.8606$ & $14.4240 \pm 0.6453$ & $12.5284 \pm 2.5280$ \\
 & \textsc{Adv.} & $0.2627 \pm 0.0189$ & $0.1251 \pm 0.0189$ & $0.1545 \pm 0.0126$ & $0.3281 \pm 0.0567$ & $0.2442 \pm 0.0578$ & $0.1818 \pm 0.0339$ \\
\midrule
\multirow{2}{*}{\textsc{GAT}} & \textsc{Perf. Gap} & $13.5540 \pm 0.4784$ & $3.6067 \pm 0.6083$ & $14.1310 \pm 0.4605$ & $11.6681 \pm 0.2751$ & $7.7704 \pm 1.2851$ & $10.2044 \pm 1.0122$ \\
 & \textsc{Adv.} & $0.2286 \pm 0.0223$ & $0.0532 \pm 0.0164$ & $0.1785 \pm 0.0186$ & $0.2550 \pm 0.0677$ & $0.2378 \pm 0.0571$ & $0.1749 \pm 0.0613$ \\
\bottomrule
\end{tabular}}
\end{table*}

\begin{table*}[htbp]
\centering
\caption{Performance gap (in \%) and membership advantage on \textsc{Pubmed} when $50\%$ of nodes are in the training set of the target model and the attacker train fraction is $r=0.1$.}
\label{tab:gap_adv_pubmed_train50_r10}
\setlength{\tabcolsep}{5pt}
\scriptsize
\resizebox{\textwidth}{!}{%
\begin{tabular}{@{}l l c c c c c c@{}}
\toprule
 &  & \multicolumn{3}{c}{\textbf{Random}} & \multicolumn{3}{c}{\textbf{Snowball}} \\
\cmidrule(lr){3-5} \cmidrule(lr){6-8}
\textbf{Model} & \textbf{Metric} & \makecell{\textsc{Original}\\\textsc{split}} & \makecell{\textsc{Full}\\\textsc{graph}} & \makecell{\textsc{No}\\\textsc{edges}} & \makecell{\textsc{Original}\\\textsc{split}} & \makecell{\textsc{Full}\\\textsc{graph}} & \makecell{\textsc{No}\\\textsc{edges}} \\
\midrule
\multirow{2}{*}{\textsc{GCN}} & \textsc{Perf. Gap} & $2.5683 \pm 0.2999$ & $0.8667 \pm 0.4915$ & $2.9532 \pm 0.3993$ & $3.9046 \pm 0.4097$ & $1.9434 \pm 0.6703$ & $1.9327 \pm 0.2369$ \\
 & \textsc{Adv.} & $0.0272 \pm 0.0080$ & $0.0079 \pm 0.0037$ & $0.0222 \pm 0.0064$ & $0.1626 \pm 0.0159$ & $0.2632 \pm 0.0085$ & $0.0446 \pm 0.0123$ \\
\midrule
\multirow{2}{*}{\textsc{GraphSage}} & \textsc{Perf. Gap} & $5.6489 \pm 0.3526$ & $3.7619 \pm 0.4879$ & $3.3778 \pm 0.3162$ & $10.4561 \pm 0.6708$ & $5.8930 \pm 0.2408$ & $3.6835 \pm 0.8331$ \\
 & \textsc{Adv.} & $0.0508 \pm 0.0085$ & $0.0305 \pm 0.0075$ & $0.0250 \pm 0.0059$ & $0.4010 \pm 0.0143$ & $0.0738 \pm 0.0108$ & $0.0411 \pm 0.0082$ \\
\midrule
\multirow{2}{*}{\textsc{GAT}} & \textsc{Perf. Gap} & $3.0837 \pm 0.3448$ & $0.8073 \pm 0.4878$ & $3.1031 \pm 0.2361$ & $4.9210 \pm 0.5259$ & $3.1546 \pm 0.5185$ & $2.1683 \pm 0.3813$ \\
 & \textsc{Adv.} & $0.0310 \pm 0.0079$ & $0.0100 \pm 0.0046$ & $0.0245 \pm 0.0067$ & $0.1718 \pm 0.0158$ & $0.0651 \pm 0.0138$ & $0.0448 \pm 0.0116$ \\
\bottomrule
\end{tabular}}
\end{table*}

\begin{table*}[htbp]
\centering
\caption{Performance gap (in \%) and membership advantage on \textsc{Pubmed} when $50\%$ of nodes are in the training set of the target model and the attacker train fraction is $r=0.2$.}
\label{tab:gap_adv_pubmed_train50_r20}
\setlength{\tabcolsep}{5pt}
\scriptsize
\resizebox{\textwidth}{!}{%
\begin{tabular}{@{}l l c c c c c c@{}}
\toprule
 &  & \multicolumn{3}{c}{\textbf{Random}} & \multicolumn{3}{c}{\textbf{Snowball}} \\
\cmidrule(lr){3-5} \cmidrule(lr){6-8}
\textbf{Model} & \textbf{Metric} & \makecell{\textsc{Original}\\\textsc{split}} & \makecell{\textsc{Full}\\\textsc{graph}} & \makecell{\textsc{No}\\\textsc{edges}} & \makecell{\textsc{Original}\\\textsc{split}} & \makecell{\textsc{Full}\\\textsc{graph}} & \makecell{\textsc{No}\\\textsc{edges}} \\
\midrule
\multirow{2}{*}{\textsc{GCN}} & \textsc{Perf. Gap} & $2.5683 \pm 0.2999$ & $0.8667 \pm 0.4915$ & $2.9532 \pm 0.3993$ & $3.9046 \pm 0.4097$ & $1.9434 \pm 0.6703$ & $1.9327 \pm 0.2369$ \\
 & \textsc{Adv.} & $0.0298 \pm 0.0059$ & $0.0070 \pm 0.0045$ & $0.0234 \pm 0.0045$ & $0.1745 \pm 0.0128$ & $0.2942 \pm 0.0168$ & $0.0530 \pm 0.0091$ \\
\midrule
\multirow{2}{*}{\textsc{GraphSage}} & \textsc{Perf. Gap} & $5.6489 \pm 0.3526$ & $3.7619 \pm 0.4879$ & $3.3778 \pm 0.3162$ & $10.4561 \pm 0.6708$ & $5.8930 \pm 0.2408$ & $3.6835 \pm 0.8331$ \\
 & \textsc{Adv.} & $0.0538 \pm 0.0031$ & $0.0338 \pm 0.0059$ & $0.0268 \pm 0.0038$ & $0.4180 \pm 0.0140$ & $0.0806 \pm 0.0135$ & $0.0474 \pm 0.0074$ \\
\midrule
\multirow{2}{*}{\textsc{GAT}} & \textsc{Perf. Gap} & $3.0837 \pm 0.3448$ & $0.8073 \pm 0.4878$ & $3.1031 \pm 0.2361$ & $4.9210 \pm 0.5259$ & $3.1546 \pm 0.5185$ & $2.1683 \pm 0.3813$ \\
 & \textsc{Adv.} & $0.0381 \pm 0.0093$ & $0.0106 \pm 0.0058$ & $0.0256 \pm 0.0048$ & $0.1904 \pm 0.0120$ & $0.0759 \pm 0.0132$ & $0.0515 \pm 0.0107$ \\
\bottomrule
\end{tabular}}
\end{table*}

\begin{table*}[htbp]
\centering
\caption{Performance gap (in \%) and membership advantage on \textsc{Pubmed} when $50\%$ of nodes are in the training set of the target model and the attacker train fraction is $r=0.5$.}
\label{tab:gap_adv_pubmed_train50_r50}
\setlength{\tabcolsep}{5pt}
\scriptsize
\resizebox{\textwidth}{!}{%
\begin{tabular}{@{}l l c c c c c c@{}}
\toprule
 &  & \multicolumn{3}{c}{\textbf{Random}} & \multicolumn{3}{c}{\textbf{Snowball}} \\
\cmidrule(lr){3-5} \cmidrule(lr){6-8}
\textbf{Model} & \textbf{Metric} & \makecell{\textsc{Original}\\\textsc{split}} & \makecell{\textsc{Full}\\\textsc{graph}} & \makecell{\textsc{No}\\\textsc{edges}} & \makecell{\textsc{Original}\\\textsc{split}} & \makecell{\textsc{Full}\\\textsc{graph}} & \makecell{\textsc{No}\\\textsc{edges}} \\
\midrule
\multirow{2}{*}{\textsc{GCN}} & \textsc{Perf. Gap} & $2.5683 \pm 0.2999$ & $0.8667 \pm 0.4915$ & $2.9532 \pm 0.3993$ & $3.9046 \pm 0.4097$ & $1.9434 \pm 0.6703$ & $1.9327 \pm 0.2369$ \\
 & \textsc{Adv.} & $0.0386 \pm 0.0057$ & $0.0099 \pm 0.0045$ & $0.0293 \pm 0.0061$ & $0.1940 \pm 0.0110$ & $0.3209 \pm 0.0105$ & $0.0593 \pm 0.0101$ \\
\midrule
\multirow{2}{*}{\textsc{GraphSage}} & \textsc{Perf. Gap} & $5.6489 \pm 0.3526$ & $3.7619 \pm 0.4879$ & $3.3778 \pm 0.3162$ & $10.4561 \pm 0.6708$ & $5.8930 \pm 0.2408$ & $3.6835 \pm 0.8331$ \\
 & \textsc{Adv.} & $0.0604 \pm 0.0037$ & $0.0402 \pm 0.0077$ & $0.0333 \pm 0.0058$ & $0.4399 \pm 0.0157$ & $0.0879 \pm 0.0140$ & $0.0495 \pm 0.0087$ \\
\midrule
\multirow{2}{*}{\textsc{GAT}} & \textsc{Perf. Gap} & $3.0837 \pm 0.3448$ & $0.8073 \pm 0.4878$ & $3.1031 \pm 0.2361$ & $4.9210 \pm 0.5259$ & $3.1546 \pm 0.5185$ & $2.1683 \pm 0.3813$ \\
 & \textsc{Adv.} & $0.0491 \pm 0.0095$ & $0.0144 \pm 0.0057$ & $0.0330 \pm 0.0061$ & $0.2070 \pm 0.0091$ & $0.0830 \pm 0.0119$ & $0.0591 \pm 0.0102$ \\
\bottomrule
\end{tabular}}
\end{table*}

\begin{table*}[htbp]
\centering
\caption{Performance gap (in \%) and membership advantage on \textsc{Pubmed} when $50\%$ of nodes are in the training set of the target model and the attacker train fraction is $r=0.8$.}
\label{tab:gap_adv_pubmed_train50_r80}
\setlength{\tabcolsep}{5pt}
\scriptsize
\resizebox{\textwidth}{!}{%
\begin{tabular}{@{}l l c c c c c c@{}}
\toprule
 &  & \multicolumn{3}{c}{\textbf{Random}} & \multicolumn{3}{c}{\textbf{Snowball}} \\
\cmidrule(lr){3-5} \cmidrule(lr){6-8}
\textbf{Model} & \textbf{Metric} & \makecell{\textsc{Original}\\\textsc{split}} & \makecell{\textsc{Full}\\\textsc{graph}} & \makecell{\textsc{No}\\\textsc{edges}} & \makecell{\textsc{Original}\\\textsc{split}} & \makecell{\textsc{Full}\\\textsc{graph}} & \makecell{\textsc{No}\\\textsc{edges}} \\
\midrule
\multirow{2}{*}{\textsc{GCN}} & \textsc{Perf. Gap} & $2.5683 \pm 0.2999$ & $0.8667 \pm 0.4915$ & $2.9532 \pm 0.3993$ & $3.9046 \pm 0.4097$ & $1.9434 \pm 0.6703$ & $1.9327 \pm 0.2369$ \\
 & \textsc{Adv.} & $0.0411 \pm 0.0123$ & $0.0203 \pm 0.0116$ & $0.0354 \pm 0.0114$ & $0.2023 \pm 0.0115$ & $0.3261 \pm 0.0121$ & $0.0667 \pm 0.0094$ \\
\midrule
\multirow{2}{*}{\textsc{GraphSage}} & \textsc{Perf. Gap} & $5.6489 \pm 0.3526$ & $3.7619 \pm 0.4879$ & $3.3778 \pm 0.3162$ & $10.4561 \pm 0.6708$ & $5.8930 \pm 0.2408$ & $3.6835 \pm 0.8331$ \\
 & \textsc{Adv.} & $0.0641 \pm 0.0109$ & $0.0449 \pm 0.0105$ & $0.0373 \pm 0.0089$ & $0.4482 \pm 0.0163$ & $0.0916 \pm 0.0167$ & $0.0629 \pm 0.0108$ \\
\midrule
\multirow{2}{*}{\textsc{GAT}} & \textsc{Perf. Gap} & $3.0837 \pm 0.3448$ & $0.8073 \pm 0.4878$ & $3.1031 \pm 0.2361$ & $4.9210 \pm 0.5259$ & $3.1546 \pm 0.5185$ & $2.1683 \pm 0.3813$ \\
 & \textsc{Adv.} & $0.0538 \pm 0.0160$ & $0.0240 \pm 0.0126$ & $0.0387 \pm 0.0106$ & $0.2215 \pm 0.0104$ & $0.0883 \pm 0.0139$ & $0.0672 \pm 0.0132$ \\
\bottomrule
\end{tabular}}
\end{table*}

%% file: references.bib
@String{Computing = "Computing" }

@String{Computer = "{IEEE} Computer" }

@String{Springer = "Springer-Verlag" }

@misc{rozemberczki2019multiscale,
            title={Multi-scale Attributed Node Embedding},
            author={Benedek Rozemberczki and Carl Allen and Rik Sarkar},
            year={2019},
            eprint={1909.13021},
            archivePrefix={arXiv},
            primaryClass={cs.LG}
        }

@inproceedings{mislove2007measurement,
  title={Measurement and analysis of online social networks},
  author={Mislove, Alan and Marcon, Massimiliano and Gummadi, Krishna P and Druschel, Peter and Bhattacharjee, Bobby},
  booktitle={Proceedings of the 7th ACM SIGCOMM conference on Internet measurement},
  pages={29--42},
  year={2007}
}

@article{kazmi2024panoramia,
  title={Panoramia: Privacy auditing of machine learning models without retraining},
  author={Kazmi, Mishaal and Lautraite, Hadrien and Akbari, Alireza and Tang, Qiaoyue and Soroco, Mauricio and Wang, Tao and Gambs, S{\'e}bastien and L{\'e}cuyer, Mathias},
  journal={Advances in Neural Information Processing Systems},
  volume={37},
  pages={57262--57300},
  year={2024}
}

@article{kennedy2021snowball,
  title={Snowball sampling study design for serosurveys early in disease outbreaks},
  author={Kennedy-Shaffer, Lee and Qiu, Xueting and Hanage, William P},
  journal={American Journal of Epidemiology},
  volume={190},
  number={9},
  pages={1918--1927},
  year={2021},
  publisher={Oxford University Press}
}

@article{sen2008collective,
  title={Collective classification in network data},
  author={Sen, Prithviraj and Namata, Galileo and Bilgic, Mustafa and Getoor, Lise and Galligher, Brian and Eliassi-Rad, Tina},
  journal={AI magazine},
  volume={29},
  number={3},
  pages={93--93},
  year={2008}
}

@inproceedings{abadi16,
author = {Abadi, Martin and Chu, Andy and Goodfellow, Ian and McMahan, H. Brendan and Mironov, Ilya and Talwar, Kunal and Zhang, Li},
title = {Deep Learning with Differential Privacy},
year = {2016},
isbn = {9781450341394},
publisher = {Association for Computing Machinery},
address = {New York, NY, USA},
url = {https://doi.org/10.1145/2976749.2978318},
doi = {10.1145/2976749.2978318},
booktitle = {Proceedings of the 2016 ACM SIGSAC Conference on Computer and Communications Security},
pages = {308–318},
numpages = {11},
location = {Vienna, Austria},
series = {CCS '16}
}

@article{papernot2018scalable,
  title={Scalable private learning with pate},
  author={Papernot, Nicolas and Song, Shuang and Mironov, Ilya and Raghunathan, Ananth and Talwar, Kunal and Erlingsson, {\'U}lfar},
  journal={arXiv preprint arXiv:1802.08908},
  year={2018}
}

@inproceedings{jayaraman2019evaluating,
  title={Evaluating differentially private machine learning in practice},
  author={Jayaraman, Bargav and Evans, David},
  booktitle={28th USENIX security symposium (USENIX security 19)},
  pages={1895--1912},
  year={2019}
}

@article{jagielski2020auditing,
  title={Auditing differentially private machine learning: How private is private sgd?},
  author={Jagielski, Matthew and Ullman, Jonathan and Oprea, Alina},
  journal={Advances in Neural Information Processing Systems},
  volume={33},
  pages={22205--22216},
  year={2020}
}

@article{hu2022membership,
  title={Membership inference attacks on machine learning: A survey},
  author={Hu, Hongsheng and Salcic, Zoran and Sun, Lichao and Dobbie, Gillian and Yu, Philip S and Zhang, Xuyun},
  journal={ACM Computing Surveys (CSUR)},
  volume={54},
  number={11s},
  pages={1--37},
  year={2022},
  publisher={ACM New York, NY}
}

@article{salem2018ml,
  title={Ml-leaks: Model and data independent membership inference attacks and defenses on machine learning models},
  author={Salem, Ahmed and Zhang, Yang and Humbert, Mathias and Berrang, Pascal and Fritz, Mario and Backes, Michael},
  journal={arXiv preprint arXiv:1806.01246},
  year={2018}
}

@inproceedings{xiang2024preserving,
  title={Preserving node-level privacy in graph neural networks},
  author={Xiang, Zihang and Wang, Tianhao and Wang, Di},
  booktitle={2024 IEEE Symposium on Security and Privacy (SP)},
  pages={4714--4732},
  year={2024},
  organization={IEEE}
}

@inproceedings{humphries2023investigating,
  title={Investigating membership inference attacks under data dependencies},
  author={Humphries, Thomas and Oya, Simon and Tulloch, Lindsey and Rafuse, Matthew and Goldberg, Ian and Hengartner, Urs and Kerschbaum, Florian},
  booktitle={2023 IEEE 36th Computer Security Foundations Symposium (CSF)},
  pages={473--488},
  year={2023},
  organization={IEEE}
}

@inproceedings{yeom2018privacy,
  title={Privacy risk in machine learning: Analyzing the connection to overfitting},
  author={Yeom, Samuel and Giacomelli, Irene and Fredrikson, Matt and Jha, Somesh},
  booktitle={2018 IEEE 31st computer security foundations symposium (CSF)},
  pages={268--282},
  year={2018},
  organization={IEEE}
}

@inproceedings{shokri2017membership,
  title={Membership inference attacks against machine learning models},
  author={Shokri, Reza and Stronati, Marco and Song, Congzheng and Shmatikov, Vitaly},
  booktitle={2017 IEEE symposium on security and privacy (SP)},
  pages={3--18},
  year={2017},
  organization={IEEE}
}

@inproceedings{duddu2020quantifying,
  title={Quantifying privacy leakage in graph embedding},
  author={Duddu, Vasisht and Boutet, Antoine and Shejwalkar, Virat},
  booktitle={MobiQuitous 2020-17th EAI International Conference on Mobile and Ubiquitous Systems: Computing, Networking and Services},
  pages={76--85},
  year={2020}
}

@article{dai2022comprehensive,
  title={A comprehensive survey on trustworthy graph neural networks: Privacy, robustness, fairness, and explainability},
  author={Dai, Enyan and Zhao, Tianxiang and Zhu, Huaisheng and Xu, Junjie and Guo, Zhimeng and Liu, Hui and Tang, Jiliang and Wang, Suhang},
  journal={arXiv preprint arXiv:2204.08570},
  year={2022}
}

@inproceedings{olatunji2021membership,
  title={Membership inference attack on graph neural networks},
  author={Olatunji, Iyiola E and Nejdl, Wolfgang and Khosla, Megha},
  booktitle={2021 Third IEEE International Conference on Trust, Privacy and Security in Intelligent Systems and Applications (TPS-ISA)},
  pages={11--20},
  year={2021},
  organization={IEEE}
}

@inproceedings{wu2021adapting,
  title={Adapting membership inference attacks to GNN for graph classification: Approaches and implications},
  author={Wu, Bang and Yang, Xiangwen and Pan, Shirui and Yuan, Xingliang},
  booktitle={2021 IEEE International Conference on Data Mining (ICDM)},
  pages={1421--1426},
  year={2021},
  organization={IEEE}
}

@article{wang2021membership,
  title={Membership inference attacks on knowledge graphs},
  author={Wang, Yu and Huang, Lifu and Yu, Philip S and Sun, Lichao},
  journal={arXiv preprint arXiv:2104.08273},
  year={2021}
}

@inproceedings{conti2022label,
  title={Label-only membership inference attack against node-level graph neural networks},
  author={Conti, Mauro and Li, Jiaxin and Picek, Stjepan and Xu, Jing},
  booktitle={Proceedings of the 15th ACM Workshop on Artificial Intelligence and Security},
  pages={1--12},
  year={2022}
}

@article{wang2024subgraph,
  title={Subgraph Structure Membership Inference Attacks against Graph Neural Networks},
  author={Wang, Xiuling and Wang, Wendy Hui},
  journal={Proceedings on Privacy Enhancing Technologies},
  year={2024}
}

@inproceedings{sanchez2020learning,
  title={Learning to simulate complex physics with graph networks},
  author={Sanchez-Gonzalez, Alvaro and Godwin, Jonathan and Pfaff, Tobias and Ying, Rex and Leskovec, Jure and Battaglia, Peter},
  booktitle={International Conference on Machine Learning},
  pages={8459--8468},
  year={2020},
  organization={PMLR}
}

@article{schulte2021integration,
  title={Integration of multiomics data with graph convolutional networks to identify new cancer genes and their associated molecular mechanisms},
  author={Schulte-Sasse, Roman and Budach, Stefan and Hnisz, Denes and Marsico, Annalisa},
  journal={Nature Machine Intelligence},
  volume={3},
  number={6},
  pages={513--526},
  year={2021},
  publisher={Nature Publishing Group}
}

@article{gaudelet2021utilizing,
  title={Utilizing graph machine learning within drug discovery and development},
  author={Gaudelet, Thomas and Day, Ben and Jamasb, Arian R and Soman, Jyothish and Regep, Cristian and Liu, Gertrude and Hayter, Jeremy BR and Vickers, Richard and Roberts, Charles and Tang, Jian and others},
  journal={Briefings in bioinformatics},
  volume={22},
  number={6},
  pages={bbab159},
  year={2021},
  publisher={Oxford University Press}
}

@inproceedings{jnaini2022powerful,
  title={How Powerful are Membership Inference Attacks on Graph Neural Networks?},
  author={Jnaini, Abdellah and Bettar, Afafe and Koulali, Mohammed Amine},
  booktitle={Proceedings of the 34th International Conference on Scientific and Statistical Database Management},
  pages={1--4},
  year={2022}
}

@article{velivckovic2017graph,
  title={Graph attention networks},
  author={Veli{\v{c}}kovi{\'c}, Petar and Cucurull, Guillem and Casanova, Arantxa and Romero, Adriana and Lio, Pietro and Bengio, Yoshua},
  journal={arXiv preprint arXiv:1710.10903},
  year={2017}
}

@inproceedings{dwork2006calibrating,
  title={Calibrating noise to sensitivity in private data analysis},
  author={Dwork, Cynthia and McSherry, Frank and Nissim, Kobbi and Smith, Adam},
  booktitle={Theory of cryptography conference},
  pages={265--284},
  year={2006},
  organization={Springer}
}

@inproceedings{he2021stealing,
  title={Stealing links from graph neural networks},
  author={He, Xinlei and Jia, Jinyuan and Backes, Michael and Gong, Neil Zhenqiang and Zhang, Yang},
  booktitle={30th USENIX security symposium (USENIX security 21)},
  pages={2669--2686},
  year={2021}
}

@inproceedings{kipf2016semi,
  author       = {Thomas N. Kipf and
                  Max Welling},
  title        = {Semi-Supervised Classification with Graph Convolutional Networks},
  booktitle    = {5th International Conference on Learning Representations, {ICLR} 2017},
  year         = {2017}
}

@article{hamilton2017inductive,
  title={Inductive representation learning on large graphs},
  author={Hamilton, Will and Ying, Zhitao and Leskovec, Jure},
  journal={Advances in neural information processing systems},
  volume={30},
  year={2017}
}

@inproceedings{sajadmanesh2021locally,
  title={Locally private graph neural networks},
  author={Sajadmanesh, Sina and Gatica-Perez, Daniel},
  booktitle={Proceedings of the 2021 ACM SIGSAC Conference on Computer and Communications Security},
  pages={2130--2145},
  year={2021}
}

@article{sajadmanesh2022gap,
  title={GAP: Differentially Private Graph Neural Networks with Aggregation Perturbation},
  author={Sajadmanesh, Sina and Shamsabadi, Ali Shahin and Bellet, Aur{\'e}lien and Gatica-Perez, Daniel},
  journal={arXiv preprint arXiv:2203.00949},
  year={2022}
}

@article{olatunji2023releasing,
  title={Releasing graph neural networks with differential privacy guarantees},
  author={Olatunji, Iyiola E and Funke, Thorben and Khosla, Megha},
  journal={Transactions on Machine Learning Research},
  year={2023}
}

@article{he2021node,
  title={Node-level membership inference attacks against graph neural networks},
  author={He, Xinlei and Wen, Rui and Wu, Yixin and Backes, Michael and Shen, Yun and Zhang, Yang},
  journal={arXiv preprint arXiv:2102.05429},
  year={2021}
}

@inproceedings{carlini2022membership,
  title={Membership inference attacks from first principles},
  author={Carlini, Nicholas and Chien, Steve and Nasr, Milad and Song, Shuang and Terzis, Andreas and Tramer, Florian},
  booktitle={2022 IEEE Symposium on Security and Privacy (SP)},
  pages={1897--1914},
  year={2022},
  organization={IEEE}
}
